\documentclass[lettersize,onecolumn]{IEEEtran}
\usepackage{amsmath,amsfonts}
\usepackage{amssymb}
\usepackage{algorithmic}
\usepackage{algorithm}
\usepackage{array}
\usepackage[caption=false,font=normalsize,labelfont=sf,textfont=sf]{subfig}
\usepackage{textcomp}
\usepackage{stfloats}
\usepackage{url}
\usepackage{verbatim}
\usepackage{graphicx}
\usepackage{booktabs}       
\usepackage{xcolor}         
\usepackage{amsthm}
\usepackage{mathtools}
\usepackage{dsfont}
\usepackage{booktabs} 
\usepackage{bbm}
\usepackage{hyperref}
\usepackage[numbers]{natbib}
\usepackage{longtable}
\usepackage{mathrsfs}

\theoremstyle{plain}
\newtheorem{theorem}{Theorem}[section]
\newtheorem{proposition}[theorem]{Proposition}
\newtheorem{lemma}[theorem]{Lemma}
\newtheorem{corollary}[theorem]{Corollary}
\theoremstyle{definition}
\newtheorem{definition}[theorem]{Definition}
\newtheorem{assumption}[theorem]{Assumption}
\theoremstyle{remark}
\newtheorem{remark}[theorem]{Remark}
\newtheorem{example}[theorem]{Example}

\begin{document}

\title{Geometric Capacity of Transformers: A Tropical Geometry Perspective}

\author{Ye~Su, Yong~Liu\textsuperscript{*}
\thanks{Ye Su is with the Shenzhen Institutes of Advanced Technology, Chinese Academy of Sciences, Shenzhen 518055, China, and also with the University of Chinese Academy of Sciences, Beijing 101407, China.}
\thanks{Yong Liu is with the Gaoling School of Artificial Intelligence, Renmin University of China, Beijing 100872, China.}
\thanks{*Corresponding authors: Yong~Liu (e-mail: liuyonggsai@ruc.edu.cn)}
}

\markboth{Under Review}%
{Shell \MakeLowercase{\textit{et al.}}: A Sample Article Using IEEEtran.cls for IEEE Journals}

\IEEEpubid{xxxx--xxxx~\copyright~xxxx xxx}

\maketitle

\begin{abstract}
To quantify the geometric capacity of transformers, we develop a tropical-geometric framework for analyzing the spatial partitions induced by conditioned self-attention. In the zero-temperature limit, we show that fixed-key top-$1$ routing is exactly represented by a power diagram in query space, while an auxiliary log-lifted value parameterization yields a vector-valued tropical rational representation. For Multi-Head Self-Attention (MHSA) with sequence length $N$ and $H$ attention heads, the joint routing geometry is encoded by Minkowski sums of headwise Newton polytopes, giving an $\mathcal{O}(N^H)$ universal bound that sharpens to $\mathcal{O}((HN)^{d_{\mathrm{model}}-1})$ once the number of heads reaches the intrinsic dimension $d_{\mathrm{model}}$. Extending this analysis across depth $L$, we derive the first tight asymptotic bounds on the number of linear regions in transformers ($\Theta\!\left(N^{\min\{H,d_{\mathrm{model}}-1\}L}\right)$). We further show that finite-temperature softmax preserves the top-$1$ routing structure and admits exponentially decaying local approximation and differential bounds away from routing boundaries.
\end{abstract}


\section{Introduction}
\IEEEPARstart{T}{he} transformer architecture \citep{vaswani2017attention} has become a central model class in natural language processing, computer vision, and scientific computing \citep{kheddar2025transformers,kashefi2026explainability}. Existing theoretical results establish strong approximation and computational properties, including universal approximation \citep{yun2020are} and Turing completeness under suitable assumptions \citep{perez2021attention,wei2022statistically}. These results answer whether transformers can represent or simulate broad classes of functions and computations. They do not, however, quantify the geometric complexity of the spatial partitions induced by self-attention.

For Continuous Piecewise Linear (CPWL) networks, the number of maximal affine regions is a classical measure of geometric capacity \citep{montufar2014number,serra2018bounding,hanin2019complexity,xiong2024number}. Extending this analysis to transformers is \textbf{nontrivial} because finite-temperature softmax is smooth rather than piecewise affine, so standard linear-region and hyperplane-arrangement methods do not apply directly. This leaves a fundamental gap in understanding \textbf{\textit{how self-attention partitions the input space topologically}}.

Tropical geometry provides a natural language for this max-affine structure \citep{zhang2018tropical,maragos2021tropical}. Through Maslov dequantization, the zero-temperature limit $\tau \to 0$ converts log-sum-exp expressions into tropical maxima. Recent studies have bridged tropical algebra with attention mechanisms, including analyses of algorithmic simulation and limiting attention representations \citep{hashemi2025tropical,alpay2026geometry}. However, tight architecture-dependent bounds on the polyhedral complexity of conditioned multi-head transformers remain unavailable. This leads to the central question of this work: \textbf{\textit{How do sequence length $N$, number of attention heads $H$, model dimension $d_{\mathrm{model}}$, and depth $L$ control the number of full-dimensional affine regions in the hard-routing skeleton of a conditioned transformer?}}

Answering this question is \textbf{non-trivial}. Quantifying the precise number of linear regions in transformers requires tracking exact geometric boundaries through successive layers. This theoretical endeavor faces four mathematical bottlenecks:
\begin{itemize}
    \item \textbf{Routing and Value Geometry:} With fixed values, zero-temperature attention becomes a piecewise-constant selector on unique-winner cells. \textbf{\textit{How can the routing geometry be separated from a value-dependent tropical representation?}}
    \item \textbf{The Multi-Head Aggregation Puzzle:} In Multi-Head Self-Attention (MHSA), parallel heads are concatenated and linearly projected. Standard hyperplane arrangement theorems, which rely on simple intersections, fail to capture this joint spatial partitioning. Formalizing \textbf{\textit{how parallel heads geometrically interleave}} remains an algebraic challenge.
    \item \textbf{The Ghost-Region Overcount:} Pairwise score-equality arrangements distinguish complete score orderings and may subdivide a top-$1$ cell without changing its winner. \textbf{\textit{How can actual routing cells be distinguished from these additional ordering cells?}}
     \item \textbf{The Finite-Temperature Reality Gap:} Hard routing defines a polyhedral skeleton, whereas finite-temperature softmax is smooth. Which routing properties remain invariant for $\tau>0$, and \textbf{\textit{how accurately does the smooth attention map approximate the hard-routing limit away from routing boundaries?}}
\end{itemize}

In this paper, we break through these bottlenecks by developing a mathematical framework based on tropical geometry to quantify the geometric capacity of transformers. By making the polyhedral structure of self-attention explicit, the framework provides a principled way to understand \textbf{\textit{how architectural choices shape the spatial partitioning complexity of transformer mappings}}. Our main contributions are as follows:

\begin{itemize}
    \item \textbf{Routing Geometry and Log-Lifted Outputs:} To separate routing from value-dependent output geometry, we prove that fixed-key top-$1$ routing cells form a power diagram determined solely by the query--key scores. We further introduce an auxiliary \textit{log-lifted} value parameterization that yields a nontrivial vector-valued tropical rational limit in the scaled log domain (\textcolor{red}{Section~\ref{sec:voronoi_routing}}).
    
    \item \textbf{Geometric Rationale for Multi-Head Attention:} To solve the aggregation puzzle, we map the MHSA to the \textit{Minkowski sum of tropical Newton polytopes}. We prove that while a single attention head is bottlenecked at $\mathcal{O}(N)$ vertices. MHSA however overcomes this limit, its complexity satisfies the universal bound $\mathcal O(N^H)$ and, when the number of heads exceeds the intrinsic dimension $d_{\mathrm{model}}$, the sharper bound $\mathcal O((HN)^{d_{\mathrm{model}}-1})$ (\textcolor{red}{Section~\ref{subsec:multi_head_geometry}}).

    \item \textbf{First Tight Topological Bounds on the Tropical Skeleton:} To avoid the ghost region overcount, we use the Voronoi equivalence to quotient out them. For the first time, we prove that the tight maximum number of regions scales as $\Theta\!\left(N^{\min\{H,d_{\mathrm{model}}-1\}L}\right)$. Importantly, by establishing matching constructive lower bounds via parabolic lifting, we guarantee that deep transformers avoid geometric collapse (\textcolor{red}{Section~\ref{subsec:linear_regions}}).
    
    \item \textbf{Geometric Stability at Finite Temperatures:} To bridge the theoretical reality gap, we prove that finite-temperature softmax preserves the top-$1$ winners and pairwise tie sets of hard routing. Away from routing boundaries, we further establish exponentially decaying bounds on the potential error, gradient error, Hessian norm, attention-output error, and output Jacobian (\textcolor{red}{Section~\ref{sec:stability}}).
\end{itemize}

\section{Related Work}

\subsection{Transformer Computational Power, Approximation, and Geometric Complexity}

Theoretical studies of transformers address several distinct notions of model capability. Formal-language and computational analyses characterize transformers as recognizers, generators, or computational models over discrete sequences, with results depending on assumptions concerning attention mechanisms, numerical precision, positional encodings, depth, and sequence length \citep{strobl2024formal,hahn2020theoretical,merrill2022saturated}. Within this line, Turing-completeness results establish computational universality under specific architectural and precision assumptions \citep{perez2021attention,wei2022statistically}, while subsequent studies examine the effects of chain-of-thought computation, padding, hard attention, and logical extensions to graph transformers \citep{merrill2023expressive,merrill2025exact,bergstrasser2024power,ahvonen2026expressive}.

A separate approximation-theoretic line treats transformers as parameterized mappings between continuous spaces. Building on classical neural approximation theory \citep{barron2002universal}, \cite{yun2020are} established universal approximation results for continuous sequence-to-sequence functions on compact domains, while later work investigated approximation properties for specific continuous tasks such as regression \citep{nath2024transformers}. These results concern the approximation of prescribed function classes and are conceptually distinct from formal-language recognition and computational universality.

Another line of research examines how individual architectural components affect the represented function class. \cite{edelman2022inductive} analyzed the sample complexity and sparse-variable creation properties of self-attention. \cite{wang2024understanding} studied the roles of dot-product attention, positional encoding, and feed-forward layers, while \cite{li2024theoretical} derived theoretical constraints for RoPE-based tensor attention. \cite{gu2026expressive} further connected transformer mappings with Maxout networks and Continuous Piecewise Linear (CPWL) functions.

Despite this progress, the continuous space-partitioning complexity of transformers remains insufficiently characterized. Formal-language and computational analyses do not address the geometry of finite-dimensional query spaces, while universal-approximation results establish representability without quantifying the combinatorial complexity of the induced partitions. Architectural studies clarify the roles of individual components but generally do not provide tight bounds on the number of full-dimensional affine regions jointly in sequence length, number of heads, model dimension, and depth. Although CPWL and Maxout formulations naturally support geometric analysis, corresponding polyhedral bounds for transformer architectures \textbf{remain underdeveloped}. Linear-region counting provides an established framework for quantifying such complexity in CPWL networks \citep{montufar2014number,serra2018bounding,hanin2019complexity,hu2020analysis,xiong2020number,xiong2024number,su2026sparsity}.

\subsection{Tropical Geometry in Deep Learning}
Tropical geometry, built on the max-plus or min-plus semiring, provides a natural framework for analyzing piecewise-linear neural networks. \cite{zhang2018tropical} provided the foundational insight that FFNs with ReLU activations are mathematically equivalent to tropical rational maps. This perspective has been extensively utilized to formalize the strict decision boundaries of deep models \citep{alfarra2022decision} and to compute the maximum number of linear regions in Multilayer Perceptrons and Convolutional Neural Networks \citep{charisopoulos2017morphological,charisopoulos2018tropical}. Beyond standard feed-forward architectures, the tropical framework has been successfully extended to analyze the expressivity and functional capabilities of Graph Neural Networks \citep{pham2024graph} and to design structured neural network compression algorithms \citep{fotopoulos2024tropnnc}. Recently, the intersection of tropical geometry and transformer architectures has garnered significant attention. \cite{hashemi2025tropical} introduced \textit{Tropical Attention} to preserve the polyhedral decision structures inherent in combinatorial dynamic programming tasks. Concurrently, studies such as the work by \cite{alpay2026geometry} analyzed standard transformers in the infinite-confidence regime (via Maslov dequantization, taking the inverse temperature $\beta \to \infty$), demonstrating that the forward pass effectively executes a Bellman-Ford shortest-path update on a latent token graph. Additionally, \cite{su2026sparsity} applied tropical geometry to quantify the combinatorial depth of Mixture-of-Experts routing mechanisms. However, while traditional applications of tropical geometry have successfully mapped out the linear regions of MLPs, the non-linear denominator of the softmax function historically complicated such analyses for standard transformers. Although recent breakthroughs \citep{hashemi2025tropical, alpay2026geometry} successfully apply tropical algebras to attention mechanisms, they predominantly focus on algorithmic simulation capabilities (e.g., executing shortest-path algorithms) or empirical out-of-distribution generalization. The strict geometric capacity of transformers in the tropical regime, specifically, \textbf{\textit{how the sequence length and multi-head mechanisms topologically partition the input space and influence the combinatorial complexity of the corresponding Newton polytopes}}, remains unquantified in the current literature.

\section{Preliminaries and Tropical Formulation}
\label{sec:preliminaries}

In this section, we establish a tropical-geometric framework for quantifying transformer geometric capacity through maximal affine-region counts. Using the tropical semiring and Maslov dequantization, we represent self-attention and related components as tropical rational maps; notation and algebraic correspondences are summarized in \textcolor{red}{Appendix~\ref{app:notations}}.

\subsection{Geometric Capacity and Linear Regions}
\label{sec:expressivity_definition}

Before introducing the algebraic framework, we define \textbf{\textit{geometric capacity}} as the maximal number of full-dimensional affine regions realizable by the analyzed model class. This notion is distinct from formal-language expressivity and approximation theory. Although linear-region counts are sometimes discussed under neural-network expressivity \citep{montufar2014number,serra2018bounding,hanin2019complexity,xiong2024number}, we use the term geometric capacity to avoid conflating these distinct notions.

\begin{definition}[Maximal Linear Region \citep{montufar2014number, serra2018bounding}]
Let $\Phi: \mathbb{R}^d \to \mathbb{R}^{d_{\text{out}}}$ be a CPWL mapping representing a neural network, where $d$ and $d_{\text{out}}$ denote the input and output dimensions, respectively. A \textit{maximal linear region} $\omega \subset \mathbb{R}^d$ is defined as a non-empty, connected open subset such that the restriction $\Phi|_{\omega}$ is an affine function (i.e., $\Phi(x) = Ax + b$ for some $A \in \mathbb{R}^{d_{\text{out}} \times d}$ and $b \in \mathbb{R}^{d_{\text{out}}}$), and for any strictly larger \textbf{connected} open set $\omega' \supsetneq \omega$, the restriction $\Phi|_{\omega'}$ is not affine.
\end{definition}

\textbf{Why Linear Regions Measure Geometric Capacity:}
The number of maximal affine regions is a fundamental structural measure of the geometric capacity of CPWL networks \citep{montufar2014number}. Within each region, the network implements an affine map, while global nonlinearity and complex decision boundaries arise from transitions across polyhedral boundaries. A larger number of regions therefore yields a finer partition of the input space and a richer piecewise-affine representation capable of capturing intricate local variations and high-frequency structure \citep{serra2018bounding,hanin2019complexity}.

\begin{definition}[Zaslavsky's Theorem]
\label{def:zaslavsky}
Let $R(n, d)$ denote the number of regions induced by an arrangement of $n$ hyperplanes in general position in a space of dimension $d$ \citep{zaslavsky1975facing}: $R(n, d) \coloneqq \sum_{j=0}^{d} \binom{n}{j}$. For $n \gg d$, this sum is dominated by the highest-order term: $R(n, d) \approx \frac{1}{d!} n^d$.
\end{definition}

\begin{example}[1-layer and 2-layer ReLU MLP of linear regions]
Consider a 2-layer ReLU MLP mapping $\Phi: \mathbb{R}^d \to \mathbb{R}^{d_{out}}$ with ambient dimension $d=2$ and $n=3$ neurons per layer. According to \textcolor{red}{Definition \ref{def:zaslavsky}}, the first hidden layer induces exactly $\mathcal{N}_1 = \sum_{j=0}^{d} \binom{n}{j} = \binom{3}{0} + \binom{3}{1} + \binom{3}{2} = 7$ maximal linear regions. In \textcolor{red}{Figure~\ref{fig:mlp_expressivity}}, deep composition facilitates recursive shattering: the second layer's hyperplanes intersect the folded image of $\mathbb{R}^d$, subdividing each pre-existing region. This multiplicative interaction yields a theoretical maximum of $\mathcal{N}_2 = \mathcal{N}_1 \cdot \sum_{j=0}^{d} \binom{n}{j} = 49$ regions. This growth provides a classical measure of the geometric capacity of a CPWL architecture, which this work extends to the non-piecewise-linear self-attention mechanism.
\end{example}

\begin{figure}[htbp]
    \centering
    \includegraphics[width=0.75\linewidth]{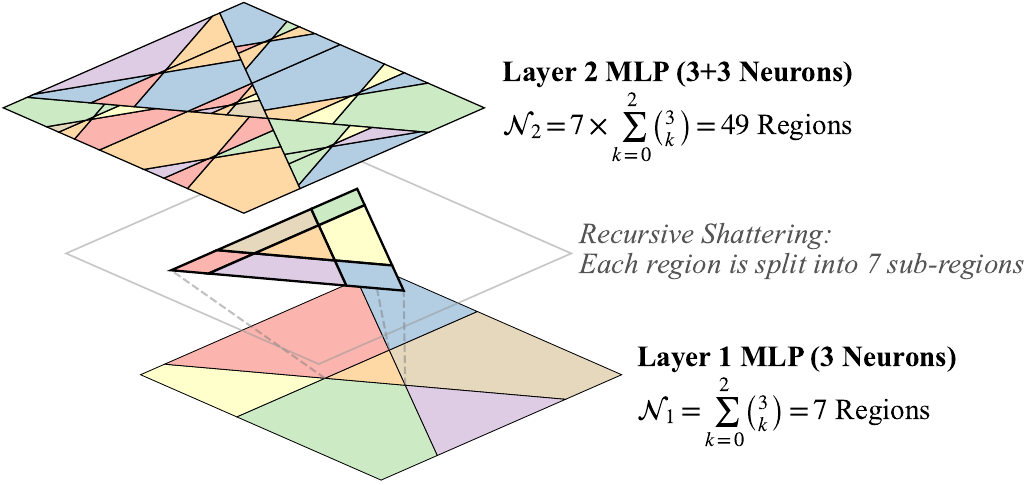}
    \caption{\textbf{Recursive spatial partitioning in a 2-layer MLP ($d=2$)}. \textbf{(Bottom)} Layer 1 ($n=3$ neurons) induces $\mathcal{N}_1 = 7$ linear regions via \textcolor{red}{Definition \ref{def:zaslavsky}}. \textbf{(Middle)} Structural space-folding: the ReLU activation enables subsequent hyperplanes to intersect the pre-activated regions. \textbf{(Top)} Layer 2 ($n=3$) achieves a multiplicative expansion, yielding a theoretical maximum of $\mathcal{N}_2 = 49$ regions.}
    \label{fig:mlp_expressivity}
\end{figure}

\textbf{The Obstruction in Transformers:} While the exact counting of linear regions has matured for classical Feed-Forward Networks (FFNs) \citep{montufar2014number,serra2018bounding,hanin2019complexity}, extending this topological metric to the transformer architecture inherently faces a mathematical challenge. The self-attention mechanism, characterized by the softmax function, is smooth and non-piecewise-linear. Therefore, it does \textbf{not trivially} induce discrete linear regions, rendering classical hyperplane arrangement theories inapplicable. 

To bridge this gap and quantify the continuous spatial partitioning capabilities of attention mechanisms, we must adopt a geometric framework capable of translating smooth log-sum-exp operations into piecewise-linear polytopes. This necessitates the introduction of \textit{Tropical Geometry} and \textit{Maslov Dequantization}.

\subsection{Tropical Algebra and Geometry}
We operate over the \textit{max-plus semiring} (also known as the tropical semiring), denoted by $\mathbb{T} = (\mathbb{R} \cup \{-\infty\}, \oplus, \otimes)$ \citep{mikhalkin2009tropical,maclagan2015introduction}.

\begin{definition}[Tropical Semiring Operations]
For any scalars $a, b \in \mathbb{T}$, the tropical addition and multiplication are defined as \citep{maragos2021tropical}: $a \oplus b := \max(a, b), \quad a \otimes b := a + b$. The identity elements for $\oplus$ and $\otimes$ are $-\infty$ and $0$, respectively.
\end{definition}

These operations extend to matrices and vectors in the standard manner. If $A, B \in \mathbb{T}^{n \times n}$, their product is $(A \otimes B)_{ij} = \bigoplus_{k=1}^n (A_{ik} \otimes B_{kj})$, which exactly corresponds to standard matrix multiplication evaluated in the $(\max, +)$ algebra \citep{viro2011basic}.

\begin{definition}[Tropical Polynomial and Newton Polytope]
A tropical polynomial $P: \mathbb{R}^d \to \mathbb{R}$ is a function of the form \citep{speyer2005tropical}: $P(x) = \bigoplus_{j=1}^m \left( c_j \otimes x^{\otimes \alpha_j} \right) = \max_{j=1,\dots,m} \left( c_j + \langle \alpha_j, x \rangle \right)$, where $x \in \mathbb{R}^d$ is the input vector, $c_j \in \mathbb{R}$ are scalar coefficients, and $\alpha_j \in \mathbb{Z}^d$ are exponent vectors. The geometric structure of $P$ is completely characterized by its \textbf{Newton Polytope}, denoted as $\text{Newt}(P)$, which is defined as the convex hull of the exponent vectors $\{\alpha_1, \dots, \alpha_m\} \subset \mathbb{R}^d$ \citep{zhang2018tropical,su2026sparsity}.
\end{definition}
Geometrically, a tropical polynomial defines a convex CPWL function whose graph forms a convex polyhedron. Projecting this graph onto the input space $\mathbb{R}^d$ yields a regular subdivision into maximal linear regions, fully determined by $\text{Newt}(P)$ and the coefficients $c_j$. Since tropical division corresponds to standard subtraction ($\oslash$), this extends to tropical rational functions, as defined in \textcolor{red}{Definition~\ref{def:trop_ration_map}}.

\begin{definition}[Tropical Rational Map]
\label{def:trop_ration_map}
A tropical rational function $F: \mathbb{R}^d \to \mathbb{R}$ is defined as the tropical quotient (i.e., standard subtraction) of two tropical polynomials, expressed as $F(x) = P(x) \oslash Q(x) = P(x) - Q(x)$ \citep{mikhalkin2009tropical,zhang2018tropical}. Extending this, a \textbf{vector-valued tropical rational map} $\mathbf{F} : \mathbb{R}^d \to \mathbb{R}^{d_{out}}$ is a mapping where each coordinate function $F_c(x)$ takes the form $P_c(x) \oslash Q_c(x)$. We define it as \textbf{structured} if all output coordinates strictly share a common denominator $Q_c(x) = Q(x)$.
\end{definition}
Geometrically, while tropical polynomials strictly represent convex CPWL functions, tropical rational maps represent Difference-of-Convex (DC) CPWL functions, which are capable of forming complex, non-convex spatial decision boundaries \citep{su2026sparsity}.

\begin{definition}[Normal Fan and Tropical Hypersurface]
For a tropical polynomial $P(x) = \max_{j \in \{1, \dots, m\}} (\langle \alpha_j, x \rangle + c_j)$, the \textbf{tropical hypersurface} $\mathcal{H}(P)$ is the set of points where the maximum is attained by at least two distinct terms \citep{maclagan2015introduction}:
\begin{equation*}
    \mathcal{H}(P) = \{ x \in \mathbb{R}^d : \exists i \neq j, \, \langle \alpha_i, x \rangle + c_i = \langle \alpha_j, x \rangle + c_j = P(x) \}.
\end{equation*}
The \textbf{normal fan} $\Sigma(P)$ is the collection of polyhedral cones $\{\sigma_j\}_{j=1}^m$ that partition $\mathbb{R}^d$, where each $d$-dimensional cone $\sigma_j$ is defined by the optimality of the $j$-th term \citep{ziegler2012lectures}:
\begin{equation*}
    \sigma_j = \{ x \in \mathbb{R}^d : \langle \alpha_j, x \rangle + c_j \ge \langle \alpha_i, x \rangle + c_i, \, \forall i \in \{1, \dots, m\} \}.
\end{equation*}
The union of all lower-dimensional faces of these cones constitutes $\mathcal{H}(P)$. For a structured rational map $\mathbf{F}$, the global decision boundary $\mathcal{H}$ is defined as the union of hypersurfaces across all coordinates: $\mathcal{H} = \bigcup_c \mathcal{H}(P_c) \cup \mathcal{H}(Q)$.
\end{definition}

\begin{definition}[Common Refinement \citep{fulton1993introduction,ziegler2012lectures}]
Given two polyhedral fans $\Sigma_1$ and $\Sigma_2$ in $\mathbb{R}^d$, their \textbf{common refinement}, denoted as $\Sigma_1 \wedge \Sigma_2$, is the fan obtained by taking all possible intersections of cones from each: $\Sigma_1 \wedge \Sigma_2 = \{ \sigma_1 \cap \sigma_2 \mid \sigma_1 \in \Sigma_1, \sigma_2 \in \Sigma_2 \}$. This operation corresponds to the superimposition of the spatial partitions induced by both fans.
\end{definition}

\subsection{Maslov Dequantization}
To bridge standard neural network operations with tropical geometry, we utilize \textit{Maslov dequantization} \citep{litvinov2005maslov}. This process treats standard arithmetic as a continuous deformation of tropical arithmetic, controlled by a temperature parameter $\tau > 0$.

\begin{definition}[Deformed Semiring Operations]
For a given temperature $\tau > 0$, we define the smooth operations: $a \oplus_\tau b := \tau \log(e^{a/\tau} + e^{b/\tau}), \quad a \otimes_\tau b := a + b$.
\end{definition}
The operator $\oplus_\tau$ is the standard \texttt{LogSumExp} (LSE) function scaled by $\tau$. It is a foundational property that $\lim_{\tau \to 0^+} (a \oplus_\tau b) = \max(a, b) = a \oplus b$. Consequently, standard matrix multiplication involving the LSE operation asymptotically converges to tropical matrix multiplication as $\tau \to 0$. This strict asymptotic boundary is referred to as the \textit{tropical limit} or \textit{zero-temperature limit}.

\subsection{The Tropical Formulation of Self-Attention}
Let $X \in \mathbb{R}^{N \times d_{\mathrm{model}}}$ denote the sequence of $N$ input tokens, where each row $x_i$ is a $d_{\mathrm{model}}$-dimensional embedding. For a single attention head, we define the query, key, and value matrices as $Q = X W_Q$, $K = X W_K$, and $V = X W_V$, where $W_Q, W_K \in \mathbb{R}^{d_{\mathrm{model}} \times d_k}$ and $W_V \in \mathbb{R}^{d_{\mathrm{model}} \times d_{v}}$ are learnable projections. Specifically, these matrices are composed of row vectors $Q = [q_1, \dots, q_N]^\top$, $K = [k_1, \dots, k_N]^\top$, and $V = [v_1, \dots, v_N]^\top$, where $q_i, k_i \in \mathbb{R}^{d_k}$ and $v_i \in \mathbb{R}^{d_v}$ represent the embeddings for the $i$-th token. Furthermore, we denote the $j$-th element of the $i$-th value vector $v_i$ as $v_{i,j}$, for $i \in \{1, \dots, N\}$ and $j \in \{1, \dots, d_v\}$.

Identifying the temperature parameter $\tau$ with the inverse scaling factor (typically $1/\sqrt{d_k}$), the $i$-th row of the standard dot-product attention output, $Z \in \mathbb{R}^{N \times d_v}$, is computed as \citep{hashemi2025tropical}:
\begin{equation*}
    z_i = \frac{\sum_{j=1}^N e^{\langle q_i, k_j \rangle / \tau} v_j}{\sum_{l=1}^N e^{\langle q_i, k_l \rangle / \tau}}.
\end{equation*}
For the purpose of spatial partitioning analysis, we represent the attention mechanism as a continuous functional mapping $Z: \mathbb{R}^{d_k} \to \mathbb{R}^{d_v}$. For any arbitrary query vector $q \in \mathbb{R}^{d_k}$, the output is given by $Z(q) = \sum_{j} \text{softmax}(\langle q, k_j \rangle / \tau) v_j$, such that the $i$-th row of the output matrix is simply $z_i = Z(q_i)$.

To analyze this under the tropical framework, we evaluate an auxiliary positive attention family in the log domain using the generalized max-plus convention, which allows real slope vectors and thus accommodates learned keys \citep{zhang2018tropical}. Since ordinary division becomes tropical subtraction $\oslash$, the log-domain output is expressed as the difference of two tropical max-affine functions.

\begin{definition}[Log-Lifted Tropical Representation of an Auxiliary Attention Family]
Fix the keys $k_1,\dots,k_N$ and coefficients $\tilde{v}_{j,c}\in\mathbb{R}$ for all $j\in\{1,\dots,N\}$ and $c\in\{1,\dots,d_v\}$, independently of $\tau$. For each output coordinate $c$ and query $q\in\mathbb{R}^{d_k}$, define the \textbf{log-lifted} positive value family $v_{j,c}^{(\tau)}=e^{\tilde{v}_{j,c}/\tau}$ and the corresponding auxiliary attention coordinate
\begin{equation*}
    \widetilde{z}_c^{(\tau)}(q):=\frac{\sum_{j=1}^N e^{\langle q,k_j\rangle/\tau}v_{j,c}^{(\tau)}}{\sum_{l=1}^N e^{\langle q,k_l\rangle/\tau}}.
\end{equation*}
As $\tau\to0$, $\tau\log\widetilde{z}_c^{(\tau)}(q)\longrightarrow P_{\text{num},c}(q)\oslash P_{\text{denom}}(q)$, where the numerator and denominator converge to distinct tropical polynomials: $P_{\text{num},c}(q):=\bigoplus_{j=1}^N\left(\langle q,k_j\rangle\otimes\tilde{v}_{j,c}\right)$ and $P_{\text{denom}}(q):=\bigoplus_{l=1}^N\langle q,k_l\rangle$. For token $i$, the corresponding coordinate is obtained by evaluating $\widetilde{z}_c^{(\tau)}$ at $q=q_i$.
\end{definition}

\begin{remark}[Log-Lifted Values and the Fixed-Value Limit]
\label{rem:theoretical_artifice}
The parameterization $v_{j,c}^{(\tau)}=e^{\tilde{v}_{j,c}/\tau}$ is a formal algebraic device assigning the coefficient the finite tropical limit $\tau\log v_{j,c}^{(\tau)}=\tilde{v}_{j,c}$. At any fixed $\tau>0$, a positive value can be written in this form by setting $\tilde{v}_{j,c}=\tau\log v_{j,c}$. In the asymptotic construction, however, $\tilde{v}_{j,c}$ is fixed as $\tau\to0$, so this defines an auxiliary $\tau$-indexed scaling regime rather than a temperature-dependent value matrix for a standard transformer.

To see why this scaling is needed, consider a temperature-independent value matrix $V$. For the coordinatewise log-domain calculation, assume $v_{j,c}>0$; signed standard values are not directly covered by this representation. Since $\log v_{j,c}$ is finite and independent of $\tau$, $\tau\log v_{j,c}\to0$. Equivalently, $e^{\langle q,k_j\rangle/\tau}v_{j,c}=e^{(\langle q,k_j\rangle+\tau\log v_{j,c})/\tau}$, and the correction $\tau\log v_{j,c}$ vanishes. Hence
\begin{equation*}
    \lim_{\tau\to0}\tau\log\left(\sum_{j=1}^N e^{\langle q,k_j\rangle/\tau}v_{j,c}\right)=\max_j\langle q,k_j\rangle.
\end{equation*}
The numerator and denominator therefore have the same tropical limit, so their tropical subtraction converges to $0$ in the scaled logarithmic domain. Thus, the finite magnitudes of a fixed value matrix disappear in this limit, although the original linear-domain attention output does not converge to zero.

To make the temperature dependence and fixed value matrix explicit, define
\begin{equation*}
    Z_V^{(\tau)}(q):=\frac{\sum_{j=1}^N e^{\langle q,k_j\rangle/\tau}v_j}{\sum_{l=1}^N e^{\langle q,k_l\rangle/\tau}}.
\end{equation*}
If $j^*(q)=\mathop{\mathrm{arg\,max}}_j\langle q,k_j\rangle$ is unique, then the softmax weights concentrate on $j^*(q)$ and $Z_V^{(\tau)}(q)\to v_{j^*(q)}$. Standard fixed-value attention therefore converges on each unique-winner region to a piecewise-constant selector. In this sense, the query--key scores determine where the selected index changes, while the value matrix determines which vector is returned on each region.

Log-lifting serves a separate purpose: it preserves the prescribed coefficient-level quantities $\tilde{v}_{j,c}$ in the tropical limit and yields a nontrivial tropical rational log-output. If the coefficients are nonconstant, the numerator may introduce channel-dependent breakpoints that refine a routing region without changing the maximizing key index or the underlying routing partition.
\end{remark}

\begin{remark}[CPWL and Polyhedral Refinement]
\label{rem:nor_fan_ref}
    Since the tropical max-affine functions $P_{\text{num},c}$ and $P_{\text{denom}}$ are polyhedral convex functions, their difference is a CPWL DC function. Let $\Sigma(P)$ denote the polyhedral complex formed by the domains of linearity of a max-affine function $P$. Geometrically, a common refinement on which every coordinate of the auxiliary structured vector-valued tropical rational map $\mathbf{F}$ is affine is given by
    \begin{equation*}
        \Sigma(\mathbf{F})=\bigwedge_{c=1}^{d_v}\left(\Sigma(P_{\text{num},c})\wedge\Sigma(P_{\text{denom}})\right).
    \end{equation*}
    The denominator complex $\Sigma(P_{\text{denom}})$ describes the zero-temperature softmax routing partition. The numerator complexes $\Sigma(P_{\text{num},c})$ may further subdivide its cells and describe additional affine breakpoints of the auxiliary lifted log-output. These numerator-induced refinements do not, in general, correspond to changes in the softmax routing index.
\end{remark}

Since $P_{\text{denom}}$ is shared across all $d_v$ output coordinates, it follows from \textcolor{red}{Definition \ref{def:trop_ration_map}} that the auxiliary positive log-lifted attention family has a structured vector-valued tropical rational limit $\mathbf{F}:\mathbb{R}^{d_k}\to\mathbb{R}^{d_v}$. By contrast, standard attention with a fixed value matrix converges on the interiors of the routing cells to a piecewise constant selector. Both constructions involve the same denominator complex $\Sigma(P_{\text{denom}})$: for standard attention, it defines the zero-temperature softmax routing partition, whereas for the auxiliary log-lifted output, it forms part of the finer affine refinement together with the numerator complexes. To extend the routing analysis to the full MHSA layer, we must formalize the aggregation of multiple independent attention heads within the tropical semiring.

\begin{definition}[Tropical Abstraction of MHSA Aggregation]
\label{def:mhsa_abstraction}
Let $x\in\mathbb{R}^{d_{\mathrm{model}}}$ denote the common pre-head query variable. For each head $h\in\{1,\dots,H\}$, let $a_j^{(h)}\in\mathbb{R}^{d_{\mathrm{model}}}$ be the effective key vector obtained by pulling the head-specific query-key score back to this common space, so that the score of key $j$ is $\langle x,a_j^{(h)}\rangle$. Define the zero-temperature routing potential
\begin{equation*}
    f_h(x):=\bigoplus_{j=1}^N\langle x,a_j^{(h)}\rangle=\max_{1\le j\le N}\langle x,a_j^{(h)}\rangle,
\end{equation*}
and let $P_h:=\operatorname{Newt}(f_h)=\operatorname{conv}\{a_1^{(h)},\dots,a_N^{(h)}\}$. The routing cell associated with index $j$ in head $h$ is
\begin{equation*}
    C_j^{(h)}:=\left\{x\in\mathbb{R}^{d_{\mathrm{model}}}:\langle x,a_j^{(h)}\rangle\geq\langle x,a_r^{(h)}\rangle,\ \forall r\in\{1,\dots,N\}\right\}.
\end{equation*}
A routing tuple $(j_1,\dots,j_H)$ is active on the possibly empty intersection $\bigcap_{h=1}^H C_{j_h}^{(h)}$. Thus, the collection of nonempty intersections forms the common refinement of the headwise routing partitions.

To encode this joint routing structure by a scalar tropical potential, define $\Psi(x):=\bigotimes_{h=1}^H f_h(x)=\sum_{h=1}^H f_h(x)$. Since tropical multiplication is ordinary addition,
\begin{equation*}
    \Psi(x)=\max_{(j_1,\dots,j_H)\in\{1,\dots,N\}^H}\left\langle x,\sum_{h=1}^H a_{j_h}^{(h)}\right\rangle.
\end{equation*}
Hence, before removing duplicate or redundant terms, the candidate slope set of $\Psi$ consists of all sums $\sum_{h=1}^H a_{j_h}^{(h)}$. Its Newton polytope is therefore the Minkowski sum of the constituent Newton polytopes \citep{gritzmann1993minkowski}:
\begin{equation*}
    \operatorname{Newt}(\Psi)=\operatorname{conv}\left\{\sum_{h=1}^H a_{j_h}^{(h)}:(j_1,\dots,j_H)\in\{1,\dots,N\}^H\right\}=P_1+\dots+P_H.
\end{equation*}
At the level of polyhedral fans, the standard identity for support functions and Minkowski sums gives
\begin{equation*}
\Sigma(\Psi)=\bigwedge_{h=1}^H\Sigma(f_h)=\bigwedge_{h=1}^H\Sigma(P_h)=\Sigma(P_1+\dots+P_H).
\end{equation*}
Thus, $\Psi$ is a scalar bookkeeping device for the joint multi-head routing partition rather than a replacement for the vector-valued MHSA output. For a fixed value matrix, the concatenated zero-temperature head output is piecewise constant on this common refinement. The subsequent linear projection $W_O$ cannot introduce additional routing boundaries, although existing cells may merge when their concatenated outputs have identical projections. The potential $\Psi$ records only the denominator-induced routing refinement; any additional numerator-induced breakpoints of an auxiliary log-lifted output must be treated separately.
\end{definition}

\subsection{Algebraic Closure of Deep Transformer Blocks}

To ensure that our topological complexity bounds (derived in \textcolor{red}{Section~\ref{sec:expressive_power}}) are mathematically well-posed, we formalize the closure properties of the remaining components in the normalization-free tropical model analyzed in this work.
\begin{itemize}
    \item \textbf{Feed-Forward Network (FFN)}, consisting of affine layers and ReLU activations, are CPWL maps. Under the tropical representation, general ReLU networks admit tropical rational representations, while tropical polynomial representations arise under additional sign constraints on the weights \citep{zhang2018tropical}. In particular, $\operatorname{ReLU}(x)=\max(0,x)=0\oplus x$ preserves piecewise-linearity.
    \item \textbf{Residual Connections} of the form $Z(X)=X+\operatorname{Layer}(X)$ preserve CPWL closure whenever $\operatorname{Layer}$ is CPWL. For tropical polynomial coordinates, ordinary addition corresponds to tropical multiplication and induces Minkowski addition of their Newton polytopes. More generally, the sum of tropical rational functions remains tropical rational by applying this operation separately to their numerator and denominator polynomials.
\end{itemize}

LayerNorm and RMSNorm are not included in the exact tropical model analyzed in this work. Both operations divide by an input-dependent square root of the feature variance or mean square and are therefore generally nonlinear and not CPWL. Consequently, the bounds in \textcolor{red}{Section~\ref{sec:expressive_power}} are stated for the normalization-free tropical transformer skeleton composed of the auxiliary tropical attention map, residual connections, and ReLU FFNs. If the normalization statistics are fixed independently of the input, the normalization becomes affine and may be absorbed into adjacent affine projections, preserving the CPWL arguments. Under this explicitly stated model restriction, each analyzed transformer block is CPWL, enabling its polyhedral regions to be composed and counted across depth.

\section{The Geometric Structure of Self-Attention}
\label{sec:voronoi_routing}

In this section, before analyzing the spatial partitioning induced by self-attention, we formalize the geometric concepts of standard and weighted tessellations from computational geometry.

\begin{definition}[Standard Voronoi Diagram \citep{gowda1983dynamic}]
\label{def:standard_voronoi}
Let $S = \{c_1, \dots, c_N\} \subset \mathbb{R}^d$ be a set of distinct generator points (sites). The \textit{Standard Voronoi Diagram} partitions $\mathbb{R}^d$ into $N$ convex polyhedral regions $\{V_1, \dots, V_N\}$. Each cell $V_j$ is defined by the standard Euclidean metric, encompassing all points $x \in \mathbb{R}^d$ that are closer to $c_j$ than to any other site $c_l$ \citep{xu2006voronoi}:
\begin{equation*}
    V_j \coloneqq \left\{ x \in \mathbb{R}^d \mid \|x - c_j\| \le \|x - c_l\|, \quad \forall l \neq j \right\}.
\end{equation*}
Geometrically, the bounding hyperplanes between adjacent cells are the exact perpendicular bisectors of the line segments joining the sites.
\end{definition}

While the standard Voronoi diagram relies on unweighted Euclidean distances, assigning individual scalar weights to the generator sites fundamentally alters the boundary geometry, yielding the Power Voronoi Diagram.

\begin{definition}[Power Voronoi Diagram]
\label{def:power_voronoi}
Let $\hat{S} = \{(c_1, w_1), \dots, (c_N, w_N)\} \subset \mathbb{R}^d \times \mathbb{R}$ be a set of generating sites with associated scalar weights. The \textit{power distance} (or Laguerre distance) from a spatial point $x \in \mathbb{R}^d$ to a weighted site is defined compactly as $d_{\text{pow}}(x, c_j) \coloneqq \|x - c_j\|^2 - w_j$. The \textit{Power Voronoi Diagram} is the cellular complex $\{P_1, \dots, P_N\}$ partitioning $\mathbb{R}^d$, where each cell $P_j$ is defined by the strict optimality of this power distance \citep{aurenhammer1987power}:
\begin{equation*}
    P_j \coloneqq \left\{ x \in \mathbb{R}^d \mid d_{\text{pow}}(x, c_j) \le d_{\text{pow}}(x, c_l), \quad \forall l \neq j \right\}.
\end{equation*}
Unlike the standard diagram, bounding hyperplanes are shifted orthogonally away from the midplanes based on weight differences. Therefore, a power cell $P_j$ can be topologically empty if its site is heavily dominated by adjacent larger weights.
\end{definition}

Building upon these strict geometric structures, we analyze input space partitioning by self-attention in the zero-temperature limit ($\tau \to 0$). We show that attention divides the projected query space into convex polyhedra, and as softmax collapses to a one-hot indicator, the routing becomes algebraically equivalent to a \textbf{Power Voronoi Diagram}, with cells determined by the key vectors' metric properties.

\begin{theorem}[Voronoi Routing Theorem]
\label{thm:voronoi}
Let $K = \{k_1, \dots, k_N\} \subset \mathbb{R}^{d_k}$ be a fixed set of key vectors. In the zero-temperature limit ($\tau \to 0$), the attention mechanism partitions the query space $\mathbb{R}^{d_k}$ into a complex of (possibly empty) convex polyhedra $\mathcal{V}(K) = \{C_1, \dots, C_N\}$. For any $q \in \text{int}(C_j)$, the attention mapping converges to a one-hot indicator for token $j$. This partition is algebraically equivalent to a \textbf{Power Voronoi Diagram} generated by sites $\{k_j\}$ with associated weights $w_j = \|k_j\|^2$. In \textcolor{red}{Figure~\ref{fig:voronoi_comparison}}, the introduction of these weights shifts the decision boundaries $\mathcal{H}$ away from standard midplanes, where the full-dimensional cells are defined by:
\begin{equation*}
    C_j = \left\{ q \in \mathbb{R}^{d_k} \mid \|q - k_j\|^2 - w_j \le \|q - k_l\|^2 - w_l, \quad \forall l \neq j \right\}.
\end{equation*}
\end{theorem}

\begin{proof}
    Please refer to \textcolor{red}{Appendix~\ref{app:proof_theorem_voronoi}} for a detailed proof.
\end{proof}

\begin{figure}[htbp]
    \centering
    \includegraphics[width=0.85\linewidth]{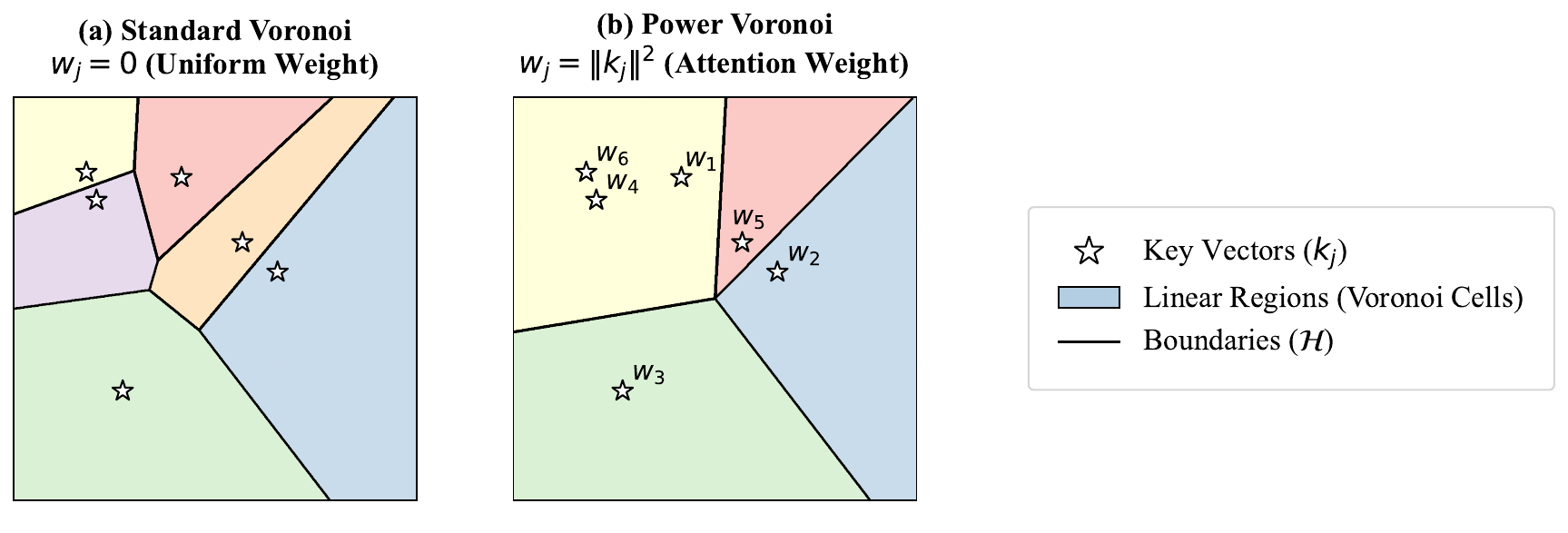}
    \caption{\textbf{Partitioning of the query space $\mathbb{R}^{d_k}$ in the tropical limit.} As detailed in the legend: white stars denote keys $k_j$, colored tiles represent maximal linear regions (Voronoi cells), and black lines identify routing boundaries $\mathcal{H}$. \textbf{(a)} Standard Voronoi diagram where routing is determined by Euclidean distance to keys. \textbf{(b)} Power Voronoi diagram corresponding to dot-product attention, where cell geometry is biased by $w_j = \|k_j\|^2$. Note the disappearance of the purple and orange regions in (b) due to their relatively small weights, confirming the possibly empty property in \textcolor{red}{Theorem~\ref{thm:voronoi}}.}
    \label{fig:voronoi_comparison}
\end{figure}

We now connect the analysis back to the standard arithmetic domain of the transformer output. As shown in \textcolor{red}{Remark~\ref{rem:nor_fan_ref}}, the tropical log-lifted model precisely captures the polyhedral geometry of the routing boundaries $\mathcal{H}$. Since the piecewise-constant output arises from the asymptotic collapse of softmax, independent of tropical algebra, we can safely replace the tropical values with the standard linear-domain value matrix $V$ within the collapsed regions.

\begin{corollary}[Piecewise Constant Value Aggregation]
\label{corr:value_aggregation}
Because the tropical rational map strictly defines the bounded polyhedral regions where the softmax collapses to an $\arg\max$ indicator, the original linear-domain attention output $Z(q) = \sum_{j} \text{softmax}(\langle q, k_j \rangle / \tau) v_j$ becomes strictly piecewise constant over the interior of the Voronoi cells. For any $q \in \text{int}(C_j)$, the original output evaluates identically to the corresponding unmodified value vector $Z(q) = v_j \in \mathbb{R}^{d_v}$. Consequently, the standard Value matrix $V$ acts purely as a piecewise constant vector field, assigning a fixed vector to each full-dimensional convex cell in the partitioned query space.
\end{corollary}

\begin{proof}
    Please refer to \textcolor{red}{Appendix~\ref{app:proof_corr_value_aggregation}} for a detailed proof.
\end{proof}

\begin{remark}[Scope of the Query-Space Partition]
\label{rem:pullback}
    \textcolor{red}{Theorem~\ref{thm:voronoi}} characterizes the partition of the projected query space $\mathbb{R}^{d_k}$ for a fixed key set $K=\{k_1,\dots,k_N\}$, rather than a direct partition of the raw sequence space $\mathbb{R}^{N\times d_{\mathrm{model}}}$. Here, $C_j$ denotes the $j$-th Power Voronoi cell defined in \textcolor{red}{Theorem~\ref{thm:voronoi}}. If the key branch is held fixed and $q_i=x_iW_Q$, then the pullback of $C_j$ to the token representation space is $\left\{x_i\in\mathbb{R}^{d_{\mathrm{model}}}:x_iW_Q\in C_j\right\}$. In the full self-attention map, however, both queries and keys depend on $X$, so the routing inequalities are coupled and generally quadratic in the entries of $X$. Hence, the conditional Power Voronoi partition in query space does not induce a fixed Cartesian-product partition of the full sequence input space.
\end{remark}

\begin{remark}[Equal-Norm Keys: Reduction to Euclidean and Angular Routing]
\label{rem:key_norm}
    Suppose the projected keys have a common Euclidean norm $\|k_j\|=r_K>0$ for all $j$, where $r_K$ denotes their shared radius from the origin. Then the Power weights in \textcolor{red}{Theorem~\ref{thm:voronoi}} are identical, $w_j=r_K^2$, and $\|q-k_j\|^2-w_j=\|q\|^2-2\langle q,k_j\rangle$. Therefore, the Power Voronoi diagram reduces to the standard Euclidean Voronoi diagram of the key sites. Moreover, for any nonzero query $q$, $\langle q,k_j\rangle=\|q\|r_K\cos\theta_j$, where $\theta_j$ is the angle between $q$ and $k_j$. Since $\|q\|r_K$ is positive and independent of $j$, dot-product routing is equivalent to maximizing cosine similarity. Thus, equal key norms remove the norm-dependent weighting, whereas variable key norms yield the general Power Voronoi geometry.
\end{remark}

\section{Expressive Power via Tropical Complexity}
\label{sec:expressive_power}

In this section, we characterize \textit{\textbf{how the sequence length $N$ amplifies the topological complexity of transformer networks}}. We first establish a geometric rationale for the MHSA mechanism via Newton polytopes, and then leverage this local complexity to derive tight asymptotic bounds on the global number of linear regions for transformers.

\subsection{The Geometric Necessity of Multi-Head Self-Attention}
\label{subsec:multi_head_geometry}

A key question in transformer design is the advantage of MHSA over a single head. We address this using the \textbf{Newton Polytope} as a combinatorial measure of partitioning capacity. As defined in \textcolor{red}{Definition~\ref{def:mhsa_abstraction}}, multi-head aggregation corresponds to the Minkowski sum of individual Newton polytopes. \textcolor{red}{Figure~\ref{fig:minkowski_mhsa}} illustrates this geometry: while a single head’s partitioning is bounded by sequence length $N$, combining multiple heads produces a complex mixed-face structure, causing the number of extreme points to grow exponentially with $H$. This topological complexity is formalized in the following \textcolor{red}{Theorem~\ref{thm:newton_polytope}}.

\begin{figure}[htbp]
    \centering
    \includegraphics[width=\textwidth]{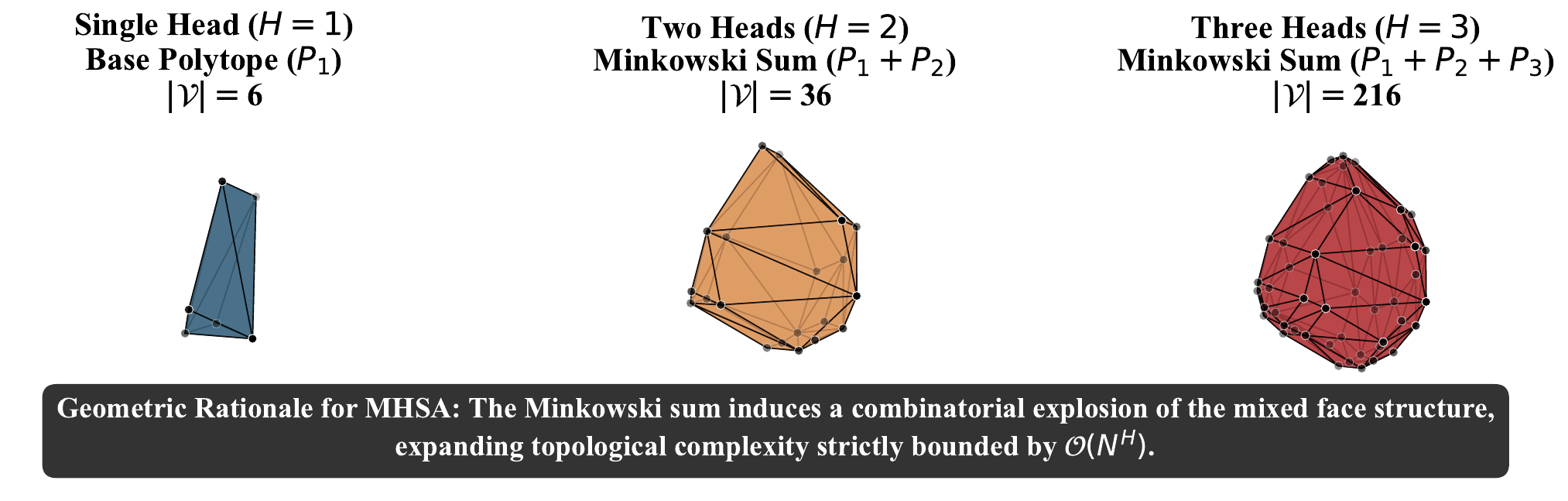}
    \caption{\textbf{Visualizing the Geometric Rationale of Multi-Head Self-Attention via Minkowski Sums.} 
    To illustrate the combinatorial polytope gain from \textcolor{red}{Theorem~\ref{thm:newton_polytope}}, we construct 3D Newton polytopes for self-attention with sequence length $N=6$. \textbf{(Left)} A single head yields a simple polytope with at most $N$ vertices. \textbf{(Middle \& Right)} Under the tropical view, multi-head aggregation corresponds to the Minkowski sum of individual polytopes, producing complex mixed-face structures and increasing the vertex count from 6 to 216. This demonstrates the superior spatial partitioning capacity of multi-head attention, scaling as $\mathcal{O}(N^H)$.}
    \label{fig:minkowski_mhsa}
\end{figure}

\begin{proposition}[Newton Polytope of Single-Head Attention]
\label{prop:single_head_polytope}
Consider a single-head attention (SHA) layer mapping to dimension $d_{\mathrm{model}}$ with sequence length $N$. The tropical potential function of this routing mechanism induces a Newton polytope strictly defined as the convex hull of the $N$ projected key vectors in $\mathbb{R}^{d_{\mathrm{model}}}$. Thus, the maximum number of extreme points (vertices) is fundamentally bottlenecked by the sequence length: $V_{\text{single}} \le N$.
\end{proposition}

While \textcolor{red}{Proposition~\ref{prop:single_head_polytope}} demonstrates that a single attention head possesses highly restricted spatial-partitioning capabilities regardless of the embedding dimension, we now prove that multi-head aggregation overcomes this linear bottleneck.

\begin{theorem}[Combinatorial Polytope Complexity of MHSA]
\label{thm:newton_polytope}
Consider an MHSA layer with $H\geq1$ attention heads, sequence length $N\geq1$, and embedding dimension $d_{\mathrm{model}}\geq2$. Assume that the routing Newton polytope associated with each head has at most $N$ vertices. Let $V_{\mathrm{multi}}$ denote the number of vertices of the joint routing Newton polytope defined by the aggregated multi-head routing potential in \textcolor{red}{Definition~\ref{def:mhsa_abstraction}}. Then, for fixed $d_{\mathrm{model}}$,
\begin{equation*}
    V_{\mathrm{multi}}=
    \begin{cases}
        \mathcal{O}(N^H), & H<d_{\mathrm{model}},\\
        \mathcal{O}\!\left((HN)^{d_{\mathrm{model}}-1}\right), & H\geq d_{\mathrm{model}},
    \end{cases}
\end{equation*}
where the hidden constants depend only on $d_{\mathrm{model}}$.
\end{theorem}

\begin{proof}
    Please refer to \textcolor{red}{Appendix~\ref{app:proof_theorem_newton_polytope}} for a detailed proof.
\end{proof}

\begin{remark}[Interpretation of the Vertex Bounds]
\label{rem:polytope_complexity_interpretation}
When $H<d_{\mathrm{model}}$, \textcolor{red}{Theorem~\ref{thm:newton_polytope}} gives $V_{\mathrm{multi}}\leq N^H$, whereas a single-head routing polytope satisfies $V_{\mathrm{single}}\leq N$. These are worst-case upper bounds and do not imply that either bound is attained.
\end{remark}

\begin{example}[SHA and MHSA under a Fixed Projection-Parameter Budget]
Following the standard multi-head attention parameterization \citep{vaswani2017attention}, and ignoring bias terms, the query, key, value, and output projections contain $2Hd_{\mathrm{model}}d_k+2Hd_{\mathrm{model}}d_v$ parameters. Under the standard choice $d_k=d_v=d_{\mathrm{model}}/H$, this count becomes $4d_{\mathrm{model}}^2$, independently of $H$. Thus, fixing $d_{\mathrm{model}}$ controls the number of attention-projection parameters, but not the parameter count of the entire transformer block. Consider $d_{\mathrm{model}}=512$ and $N=512$. Both SHA and MHSA then contain $4d_{\mathrm{model}}^2=1{,}048{,}576$ attention-projection parameters.
\begin{itemize}
    \item \textbf{SHA ($H=1$):} Setting $d_k=d_v=d_{\mathrm{model}}$, the routing is governed by a single Newton polytope. By \textcolor{red}{Proposition~\ref{prop:single_head_polytope}}, its vertex count satisfies $V_{\mathrm{single}}\leq N=512$.
    \item \textbf{MHSA ($H=8$):} Setting $d_k=d_v=d_{\mathrm{model}}/8$, the joint routing polytope is the Minkowski sum of eight head polytopes. Since $H=8<d_{\mathrm{model}}=512$, \textcolor{red}{Theorem~\ref{thm:newton_polytope}} gives the worst-case asymptotic bound $V_{\mathrm{multi}}=\mathcal{O}(N^8)=\mathcal{O}(512^8)$.
\end{itemize}
Therefore, under the same attention-projection parameter count, the multi-head upper bound has a higher polynomial dependence on the sequence length than the single-head bound. This is a worst-case asymptotic comparison and does not assert that a particular parameter configuration realizes $N^H$ vertices.
\end{example}

\subsection{Counting Linear Regions of Tropical Transformers}
\label{subsec:linear_regions}

Building on \textcolor{red}{Theorem~\ref{thm:newton_polytope}}, we study hard-routing partitions in a conditioned single-query input space, where the context keys and $W_Q$ are fixed while $x\in\mathbb{R}^{d_{\mathrm{model}}}$ varies. For geometric region counting, identify indices that induce identical score functionals, and let $M\leq N$ denote the number of distinct functionals, relabeled as $s_j(x):=\langle xW_Q,k_j\rangle,\quad j=1,\dots,M$.
The unique-winner top-$1$ routing region of key $j$ is $\widetilde C_j^\circ:=\{x:s_j(x)>s_l(x)\ \text{for all }l\neq j\}$.

\begin{definition}[Ghost-Region Overcount]
Let $S(x):=(s_1(x),\dots,s_M(x))$ be the affine score map, and let
$\mathcal H_{ab}:=\{x:s_a(x)=s_b(x)\}=S^{-1}\!\left(\{z:z_a=z_b\}\right)$, $1\leq a<b\leq M$, be its pairwise score-equality hyperplanes, with coincident hyperplanes allowed. Their arrangement is a possibly nonessential pullback of the braid arrangement, and its chambers correspond to realizable strict score orderings \citep{stanley2007introduction}. For the unique-winner region $\widetilde C_j^\circ:=\{x:s_j(x)>s_l(x),\ \forall l\neq j\}$, let $r_j$ denote the number of arrangement chambers contained in $\widetilde C_j^\circ$. The ghost-region overcount is $G:=\sum_{j=1}^M\max\{r_j-1,0\}$.
\end{definition}

The quantity $G$ counts subdivisions of a top-$1$ routing cell caused solely by changes in the ordering of nonmaximal scores. For a fixed winner $j$, the remaining $M-1$ scores admit at most $(M-1)!$ strict orderings, so $r_j\leq(M-1)!, \quad 0\leq G\leq M!-M$. If $M\leq N$, then also $G\leq N!-N$. Equality $G=M!-M$ requires all $M!$ strict score orderings to be realizable. In general, $G$ depends on the rank and geometry of $S$, so no universal prevalence is claimed. The Power Voronoi formulation counts the nonempty top-$1$ routing cells directly and therefore avoids this ordering-based overcount \citep{aurenhammer1987power}.

\begin{assumption}[Cellwise Noncollapse]
\label{ass:cellwise_noncollapse}
Let $\Phi_\ell$ denote the composition of the first $\ell$ transformer blocks, with $\Phi_0=\operatorname{id}$. For every $\ell\in\{0,\dots,L-1\}$, every full-dimensional affine cell $Q$ of $\Phi_\ell$, and every nonzero affine functional $g$ defining a routing-tie or ReLU activation boundary of block $\ell+1$, assume that $g\circ\left.\Phi_\ell\right|_Q\not\equiv0$. This cellwise noncollapse condition is an affine analogue of the standard transversality principle \citep{hirsch2012differential}; the corresponding pullback of activation hyperplanes through preceding affine regions is standard in linear-region analyses of piecewise-linear networks \citep{serra2018bounding}.
\end{assumption}

\begin{theorem}[Upper Bound on Full-Dimensional Linear Regions]
\label{thm:linear_regions}
Consider a tropical transformer network $\mathcal T(L,H,d_{\mathrm{model}},d_{\mathrm{ff}})$ in the conditioned single-query setting, where $L\geq1$ is the number of transformer blocks, $H\geq1$ is the number of attention heads per block, $d_{\mathrm{model}}\geq2$ is the embedding dimension, $d_{\mathrm{ff}}\geq1$ is the hidden width of each feed-forward network, and $N\geq1$ is the sequence length. Under \textcolor{red}{Assumption~\ref{ass:cellwise_noncollapse}}, for fixed $d_{\mathrm{model}}$, $H$, $d_{\mathrm{ff}}$, and $L$, the number of full-dimensional affine regions satisfies
\begin{equation*}
    \mathcal N\!\left(\mathcal T(L,H,d_{\mathrm{model}},d_{\mathrm{ff}})\right)=
    \begin{cases}
        \mathcal O(N^{HL}), & H<d_{\mathrm{model}},\\
        \mathcal O\!\left(N^{(d_{\mathrm{model}}-1)L}\right), & H\geq d_{\mathrm{model}}.
    \end{cases}
\end{equation*}
\end{theorem}

\begin{proof}
    Please refer to \textcolor{red}{Appendix~\ref{app:proof_theorem_linear_regions}} for a detailed proof.
\end{proof}

\begin{remark}[Interpretation of the Linear-Region Bound]
\label{rem:linear_region_interpretation}
\textcolor{red}{Theorem~\ref{thm:linear_regions}} shows that the dependence on the sequence length $N$ is governed by the joint MHSA routing partition. For fixed $d_{\mathrm{model}}$ and $d_{\mathrm{ff}}$, the FFN contributes only an $N$-independent multiplicative factor within each block, while composition across $L$ blocks multiplies the per-block sequence-length exponent by $L$. Consequently, the exponent is $HL$ when $H<d_{\mathrm{model}}$ and saturates at $(d_{\mathrm{model}}-1)L$ when $H\geq d_{\mathrm{model}}$.
\end{remark}

While \textcolor{red}{Theorem~\ref{thm:linear_regions}} gives an upper bound, repeated composition attains this sequence-length exponent when the attention routing cells are mapped onto a common invariant domain. The construction below combines coordinatewise attention routing with the folding mechanism used in classical linear-region lower bounds for deep piecewise-linear networks \citep{montufar2014number,telgarsky2016benefits}.

\begin{theorem}[Constructive Lower Bound on Linear Regions]
\label{thm:lower_bound_linear_regions}
Consider the tropical transformer network $\mathcal T(L,H,d_{\mathrm{model}},d_{\mathrm{ff}})$ in the conditioned single-query setting of \textcolor{red}{Theorem~\ref{thm:linear_regions}}, with $d_{\mathrm{model}}\geq2$ and $d_{\mathrm{ff}}\geq2\min\{H,d_{\mathrm{model}}-1\}$. Assume that at least $\min\{H,d_{\mathrm{model}}-1\}$ attention heads have query--key dimension at least two. Then there exists a parameter configuration such that, for fixed $d_{\mathrm{model}}$, $H$, $d_{\mathrm{ff}}$, and $L$, the number of full-dimensional affine regions satisfies
\begin{equation*}
    \mathcal N\!\left(\mathcal T(L,H,d_{\mathrm{model}},d_{\mathrm{ff}})\right)=
    \begin{cases}
        \Omega(N^{HL}), & H<d_{\mathrm{model}},\\
        \Omega\!\left(N^{(d_{\mathrm{model}}-1)L}\right), & H\geq d_{\mathrm{model}}.
    \end{cases}
\end{equation*}
\end{theorem}

\begin{proof}
    Please refer to \textcolor{red}{Appendix~\ref{app:proof_theorem_lower_bound}} for a detailed constructive proof.
\end{proof}

\begin{remark}[Geometric Mechanism of the Construction]
\label{rem:difficulty_lower_bound}
The construction assigns one active attention head to each of $\min\{H,d_{\mathrm{model}}-1\}$ coordinates. Homogeneous parabolic routing in the two-dimensional planes spanned by each active coordinate and one shared anchor coordinate produces $N^{\min\{H,d_{\mathrm{model}}-1\}}$ disjoint routing branches per block. On each branch, the attention residual translates a small full-dimensional box to a common location, and the ReLU FFN realizes a branch-independent affine expansion onto the same invariant domain. Repeating this full-branch mechanism across depth yields the exponent $HL$ when $H<d_{\mathrm{model}}$ and $(d_{\mathrm{model}}-1)L$ when $H\geq d_{\mathrm{model}}$.
\end{remark}

The construction therefore shows that the upper-bound exponent in \textcolor{red}{Theorem~\ref{thm:linear_regions}} is attainable within the conditioned single-query, zero-temperature model. The residual connection and the FFN jointly realize a surjective folding map on an invariant affine slice, preventing the loss of regions under repeated composition.

\begin{corollary}[Asymptotic Tightness in Sequence Length]
\label{corr:asymptotic_tightness}
Let $\mathcal N_{\max}(N)$ denote the supremum, over all admissible parameter configurations of $\mathcal T(L,H,d_{\mathrm{model}},d_{\mathrm{ff}})$, of the number of full-dimensional affine regions in the conditioned single-query setting. Under the assumptions of \textcolor{red}{Theorem~\ref{thm:lower_bound_linear_regions}}, for fixed $d_{\mathrm{model}}$, $H$, $d_{\mathrm{ff}}$, and $L$,
\begin{equation*}
    \mathcal N_{\max}(N)=
    \begin{cases}
        \Theta(N^{HL}), & H<d_{\mathrm{model}},\\
        \Theta\!\left(N^{(d_{\mathrm{model}}-1)L}\right), & H\geq d_{\mathrm{model}},
    \end{cases}
    \quad \text{as } N\to\infty.
\end{equation*}
Thus, the sequence-length exponents in \textcolor{red}{Theorem~\ref{thm:linear_regions}} are asymptotically tight in both the unsaturated and saturated head regimes. 
\end{corollary}

\section{Geometric Stability at Finite Temperatures}
\label{sec:stability}

In this section, we address the following question: \textbf{\textit{To what extent does finite-temperature attention remain close to its zero-temperature tropical limit?}} \textcolor{red}{Sections~\ref{sec:voronoi_routing}} and~\ref{sec:expressive_power} analyze the exact polyhedral model obtained as $\tau\to0$, whereas attention at every fixed $\tau>0$ is smooth. We therefore study the finite--zero temperature discrepancy on subsets of routing cells separated from score ties by a positive margin. Throughout this section, we use the convention $p_j^{(\tau)}(q):=e^{\langle q,k_j\rangle/\tau}/\sum_{r=1}^{N}e^{\langle q,k_r\rangle/\tau}$. Accordingly, if $s_j(q):=\langle q,k_j\rangle/\sqrt{d_k}$ denotes the scaled attention score, standard softmax corresponds to $\tau=1$; equivalently, when $\langle q,k_j\rangle$ denotes the unscaled dot product, the standard factor $1/\sqrt{d_k}$ corresponds to $\tau=\sqrt{d_k}$ \citep{vaswani2017attention}.

A function-value bound alone is insufficient to establish geometric proximity to the tropical model, because a smooth function may remain uniformly close to an affine function while exhibiting large or rapidly varying derivatives. We therefore control successively stronger quantities: exact preservation of the dominant routing index identifies the same routing partition; the potential-value bound measures zero-order error; the gradient and Hessian bounds control slope mismatch and curvature; the local affine estimate quantifies the first-order approximation error within each margin-trimmed cell; and the attention-output and Jacobian bounds transfer these conclusions from the scalar LSE potential to the actual vector-valued attention map. Together, these estimates provide a differential, rather than merely pointwise, comparison between finite-temperature attention and its zero-temperature tropical limit.

\begin{theorem}[Finite-Temperature Approximation of Tropical Routing]
\label{thm:finite_temperature_stability}
Let $N\geq2$, and let $k_1,\dots,k_N\in\mathbb R^{d_k}$ and $v_1,\dots,v_N\in\mathbb R^{d_v}$ be fixed key and value vectors. For a query $q\in\mathbb R^{d_k}$ and temperature $\tau>0$, define the softmax weight $p_j^{(\tau)}(q)$, log-sum-exp potential $P^{(\tau)}(q)$, and attention output $A^{(\tau)}(q)$ by
\begin{equation*}
    p_j^{(\tau)}(q):=\frac{e^{\langle q,k_j\rangle/\tau}}{\sum_{r=1}^{N}e^{\langle q,k_r\rangle/\tau}},\quad
    P^{(\tau)}(q):=\tau\log\left(\sum_{j=1}^{N}e^{\langle q,k_j\rangle/\tau}\right),\quad
    A^{(\tau)}(q):=\sum_{j=1}^{N}p_j^{(\tau)}(q)v_j.
\end{equation*}
Let $P^{(0)}(q):=\max_{1\leq j\leq N}\langle q,k_j\rangle$ denote the zero-temperature potential. For $i\in\{1,\dots,N\}$ and margin $\delta>0$, define the margin-trimmed routing cell $C_i^\delta$, together with the maximal key $K_i$ and value deviations $D_i$ by. $C_i^\delta:=\left\{q\in\mathbb R^{d_k}:\langle q,k_i-k_j\rangle\geq\delta\ \text{for every }j\neq i\right\}$, $K_i:=\max_{j\neq i}\|k_j-k_i\|_2$, $D_i:=\max_{j\neq i}\|v_j-v_i\|_2$. Then the following statements hold.
\begin{itemize}
    \item \textbf{Exact Temperature-Invariance of Top-$1$ Routing.} Softmax preserves the ordering of the attention scores, so the dominant routing indices, routing cells, and tie hyperplanes are identical at every positive temperature and in the zero-temperature limit:
    \begin{equation*}
        \operatorname*{arg\,max}_{1\leq j\leq N}p_j^{(\tau)}(q)=\operatorname*{arg\,max}_{1\leq j\leq N}\langle q,k_j\rangle,\quad q\in\mathbb R^{d_k},\quad \tau>0.
    \end{equation*}

    \item \textbf{Potential-Value Approximation.} On a margin-trimmed routing cell, the finite-temperature LSE potential differs from its tropical affine branch by an exponentially small nonnegative error:
    \begin{equation*}
        0\leq P^{(\tau)}(q)-P^{(0)}(q)\leq\tau\log\left(1+(N-1)e^{-\delta/\tau}\right),\quad q\in C_i^\delta.
    \end{equation*}

    \item \textbf{Gradient and Hessian Approximation.} On $C_i^\delta$, the gradient approaches the winning key $k_i$, while the curvature of the finite-temperature potential decays exponentially with the normalized margin $\delta/\tau$:
    \begin{equation*}
        \left\|\nabla P^{(\tau)}(q)-k_i\right\|_2\leq K_i(N-1)e^{-\delta/\tau},\quad \left\|\nabla^2P^{(\tau)}(q)\right\|_2\leq\frac{K_i^2}{\tau}(N-1)e^{-\delta/\tau}.
    \end{equation*}

    \item \textbf{Local Affine Approximation of the Potential.} The exponentially small Hessian makes the gradient nearly constant and bounds the deviation from the first-order affine approximation quadratically in the displacement:
    \begin{equation*}
        \left\|\nabla P^{(\tau)}(q')-\nabla P^{(\tau)}(q)\right\|_2\leq\frac{K_i^2}{\tau}(N-1)e^{-\delta/\tau}\|q'-q\|_2,\quad q,q'\in C_i^\delta,
    \end{equation*}
    \begin{equation*}
        \left|P^{(\tau)}(q')-P^{(\tau)}(q)-\left\langle\nabla P^{(\tau)}(q),q'-q\right\rangle\right|\leq\frac{K_i^2}{2\tau}(N-1)e^{-\delta/\tau}\|q'-q\|_2^2.
    \end{equation*}

    \item \textbf{Attention-Output Approximation.} The finite-temperature attention output approaches the zero-temperature value selector $v_i$, and its Jacobian with respect to the query becomes exponentially small:
    \begin{equation*}
        \left\|A^{(\tau)}(q)-v_i\right\|_2\leq D_i(N-1)e^{-\delta/\tau},\quad \left\|J A^{(\tau)}(q)\right\|_2\leq\frac{2D_iK_i}{\tau}(N-1)e^{-\delta/\tau},\quad q\in C_i^\delta,
    \end{equation*}
    where $J A^{(\tau)}(q)$ denotes the Jacobian of $A^{(\tau)}$ with respect to $q$.

    \item \textbf{Local Stability of the Attention Output.} Consequently, the attention output varies only by an exponentially small Lipschitz factor between any two queries in the same margin-trimmed routing cell:
    \begin{equation*}
        \left\|A^{(\tau)}(q')-A^{(\tau)}(q)\right\|_2\leq\frac{2D_iK_i}{\tau}(N-1)e^{-\delta/\tau}\|q'-q\|_2,\quad q,q'\in C_i^\delta.
    \end{equation*}
\end{itemize}
\end{theorem}

\begin{proof}
    Please refer to \textcolor{red}{Appendix~\ref{app:proof_theorem_finite_temperature_stability}} for a detailed proof.
\end{proof}

Because softmax is strictly order-preserving, the first conclusion of \textcolor{red}{Theorem~\ref{thm:finite_temperature_stability}} is exact rather than asymptotic: for any indices $j$ and $r$, $p_j^{(\tau)}(q)>p_r^{(\tau)}(q)$ if and only if $e^{\langle q,k_j\rangle/\tau}>e^{\langle q,k_r\rangle/\tau}$, which holds if and only if $\langle q,k_j\rangle>\langle q,k_r\rangle$. Hence finite temperature does not alter the top-$1$ Power Voronoi routing partition or its pairwise tie sets; it only replaces the hard one-hot selector by a smooth probability distribution whose mass outside the winning index is exponentially small on $C_i^\delta$. On this margin-trimmed cell, the tropical potential is the affine function $P^{(0)}(q)=\langle q,k_i\rangle$, with gradient $k_i$ and zero Hessian, while the zero-temperature attention output is the constant vector $A^{(0)}(q)=v_i$. Thus, for every fixed $\delta>0$, \textcolor{red}{Theorem~\ref{thm:finite_temperature_stability}} establishes uniform $C^2$ convergence of the routing potential and uniform $C^1$ convergence of the attention output to their zero-temperature counterparts on $C_i^\delta$ as $\tau\to0$. The theorem does not assert that the finite-temperature map itself is piecewise affine, since for every $\tau>0$ the LSE potential and softmax attention output are generally non-affine on every open set. Consequently, the exact linear-region counts of \textcolor{red}{Section~\ref{sec:expressive_power}} apply to the tropical skeleton, whereas \textcolor{red}{Theorem~\ref{thm:finite_temperature_stability}} quantifies the analytic discrepancy between that skeleton and the finite-temperature model.

In summary, the finite-temperature analysis therefore establishes that the tropical model is not merely a pointwise limiting abstraction. Its top-$1$ routing geometry is exactly shared by finite-temperature softmax, and its affine routing potential and piecewise-constant value field approximate the finite-temperature counterparts with exponentially small errors on margin-separated cell interiors.

\section{Conclusion}
\label{sec:conclusion}

We introduced a tropical geometry framework to quantify the geometric capacity of transformer. By proving an equivalence between zero-temperature self-attention and Power Voronoi diagram, for the first time, we show that the tight maximal number of linear regions scales as $\Theta\!\left(N^{\min\{H,d_{\mathrm{model}}-1\}L}\right)$, with per-layer complexity bounded by $\mathcal{O}(N^H)$ and sharpening to $\mathcal{O}((HN)^{d_{\mathrm{model}}-1})$ through the Minkowski sum of multi-head Newton polytopes. Stability bounds further show that top-$1$ winners and pairwise tie sets are preserved at finite temperatures, with exponentially decaying local approximation bounds away from routing boundaries. While this provides a clear geometric measure of maximal polyhedral complexity, the work does not capture the non-piecewise-linear effects of standard normalizations or data-dependent optimization dynamics, which remain important directions for future work.

\bibliography{reference}
\bibliographystyle{IEEEtran}

\section*{Acknowledgments}
This research was supported by Beijing Natural Science Foundation (Z250001).

{\appendices
\section{Summary of Notations and Algebraic Mappings}
\label{app:notations}

To facilitate a rigorous bridging between standard DL operations, tropical algebra, and their underlying geometric interpretations, we summarize the core notations used throughout this paper in \textcolor{blue}{Table~\ref{tab:notation_summary}} and the algebraic mappings in \textcolor{blue}{Table~\ref{tab:algebraic_mapping}}. 

\begin{longtable}{p{0.18\textwidth} p{0.62\textwidth} p{0.15\textwidth}}
\caption{Summary of Mathematical Notations and Network Variables}
\label{tab:notation_summary} \\
\toprule
\textbf{Symbol} & \textbf{Definition} & \textbf{Domain/Space} \\
\midrule
\endfirsthead

\multicolumn{3}{c}%
{{\bfseries \tablename\ \thetable{} -- continued from previous page}} \\
\toprule
\textbf{Symbol} & \textbf{Definition} & \textbf{Domain/Space} \\
\midrule
\endhead

\midrule \multicolumn{3}{r}{{Continued on next page}} \\ \midrule
\endfoot

\bottomrule
\endlastfoot

\multicolumn{3}{l}{\textit{\textbf{1. Network Configurations and Global Dimensions}}} \\
\midrule
$N$ & Sequence length (number of tokens / keys / Voronoi sites) & $\mathbb{Z}^+$ \\
$d_{\mathrm{model}}$ & Ambient embedding dimension of the transformer & $\mathbb{Z}^+$ \\
$d_k, d_v$ & Projected dimensions for Query/Key and Value vectors & $\mathbb{Z}^+$ \\
$d, d_{\text{out}}$ & Generic input and output dimensions in mappings & $\mathbb{Z}^+$ \\
$H$ & Number of parallel attention heads in MHSA & $\mathbb{Z}^+$ \\
$d_{\text{ff}}, n$ & FFN hidden dimension; number of neurons or hyperplanes & $\mathbb{Z}^+$ \\
$L$ & Network depth (number of stacked transformer blocks) & $\mathbb{Z}^+$ \\
$\mathcal{T}$ & Transformer network parameterized as $(L, H, d_{\mathrm{model}}, d_{\text{ff}})$ & - \\
$\mathcal{N}(\mathcal{T})$ & Maximum number of full-dimensional linear regions & $\mathbb{Z}^+$ \\
$V_{\text{single}}, V_{\text{multi}}$ & Vertex complexity of SHA and MHSA Newton polytopes & $\mathbb{Z}^+$ \\
$\Theta, \Omega, \mathcal{O}$ & Standard asymptotic notations (Tight, Lower, Upper bounds) & - \\

\midrule
\multicolumn{3}{l}{\textit{\textbf{2. Input, Projections, and Feature Spaces}}} \\
\midrule
$X$ & Input sequence embedding matrix & $\mathbb{R}^{N \times d_{\mathrm{model}}}$ \\
$x, q$ & Single query embedding vector or query point & $\mathbb{R}^{d_{\mathrm{model}}}$ \\
$W_Q, W_K$ & Learnable linear projection matrices for Queries and Keys & $\mathbb{R}^{d_{\mathrm{model}} \times d_k}$ \\
$W_V$ & Learnable linear projection matrix for Values & $\mathbb{R}^{d_{\mathrm{model}} \times d_v}$ \\
$W_O$ & Output projection matrix for multi-head feature aggregation & $\mathbb{R}^{d_{\mathrm{model}} \times d_{\mathrm{model}}}$ \\
$Q, K, V$ & Projected query, key, and value matrices & $\mathbb{R}^{N \times d_k/d_v}$ \\
$q_i, k_j, v_j$ & The $i$-th query, $j$-th key, and $j$-th value vectors & $\mathbb{R}^{d_k}, \mathbb{R}^{d_v}$ \\
$W_Q^{(h)}, \tilde{k}_j^{(h)}$ & Head-specific query matrix and its adjoint projected key & $\mathbb{R}^{d_{\mathrm{model}} \times d_k}, \mathbb{R}^d$ \\
$a_j$ & Composite key vector in proofs: $a_j := (W_Q^{(h)})^\top k_j$ & $\mathbb{R}^{d_{\mathrm{model}}}$ \\
$Z(q; \tau)$ & Attention output mapping for query $q$ at temperature $\tau$ & $\mathbb{R}^{d_v}$ \\
$v_{j,c}$ & element of value matrix at position $(j, c)$ & $\mathbb{R}$ \\
$\langle \cdot, \cdot \rangle, \|\cdot\|_2$ & Standard inner product and Euclidean ($L_2$) norm & - \\

\midrule
\multicolumn{3}{l}{\textit{\textbf{3. Temperature, Log-Lifting, and Stability Variables}}} \\
\midrule
$\tau$ & Temperature parameter for Maslov dequantization & $\mathbb{R}_{>0}$ \\
$s, s_j$ & Attention logit vector and components $s_j = \langle q, k_j \rangle$ & $\mathbb{R}^N, \mathbb{R}$ \\
$S_l$ & Inner product score for the $l$-th token in proofs & $\mathbb{R}$ \\
$P^{(\tau)}(s)$ & Smoothed LogSumExp (LSE) potential function & $\mathbb{R}$ \\
$P^{(0)}(s)$ & Strict tropical limit (max function) of the LSE potential & $\mathbb{R}$ \\
$\tilde{v}_{j,c}$ & Log-lifted value satisfying $v_{j,c} = \exp(\tilde{v}_{j,c}/\tau)$ & $\mathbb{R}$ \\
$p, p_j$ & Softmax probability vector and $j$-th mass & $\Delta^{N-1}$ \\
$\delta, \mathcal{R}_i^\delta$ & Logit margin and the corresponding $\delta$-stable region & $\mathbb{R}_{>0}, \mathbb{R}^N$ \\
$\nabla^2 P^{(\tau)}(s)$ & Hessian matrix of the smoothed LSE potential & $\mathbb{R}^{N \times N}$ \\

\midrule
\multicolumn{3}{l}{\textit{\textbf{4. Tropical Geometry and Topological Structures}}} \\
\midrule
$\mathbb{T}, \oplus, \otimes$ & Tropical semiring (max-plus) and its addition and multiplication & - \\
$\oslash, \otimes_\tau, \oplus_\tau$ & Tropical division and deformed smooth operations & - \\
$P(x)$ & Tropical polynomial $P(x) = \max_j (\langle \alpha_j, x \rangle + c_j)$ & $\mathbb{R}^d \to \mathbb{R}$ \\
$\text{Newt}(P)$ & Newton polytope: convex hull of exponent vectors $\{\alpha_j\}$ & $\text{conv}(\mathbb{R}^d)$ \\
$\Sigma$ & \textbf{Normal Fan} & - \\
$\Sigma(P)$ & \textbf{Normal Fan} of $\text{Newt}(P)$, representing the polyhedral subdivision & $\{\sigma_j\} \subset \mathbb{R}^d$ \\
$\sigma_j$ & A $d$-dimensional polyhedral cone in the fan $\Sigma(P)$ & $\mathbb{R}^d$ \\
$\mathcal{H}(P), \mathcal{H}$ & Tropical hypersurface (locus of non-smoothness) and global boundary & - \\
$\Sigma_1 \wedge \Sigma_2$ & \textbf{Common Refinement}: superimposition of two polyhedral fans & - \\
$f_h(q), P_h$ & Tropical potential and Newton polytope of the $h$-th head & - \\
$\Psi$ & \textbf{Aggregate Potential} of MHSA: $\Psi := \bigotimes_{h=1}^H f_h = \sum f_h$ & $\mathbb{R}^d \to \mathbb{R}$ \\
$\text{Newt}(\Psi)$ & \textbf{Minkowski Sum} of individual head Newton polytopes: $\sum P_h$ & $\text{conv}(\mathbb{R}^d)$ \\
$\mathbf{F}, \Phi$ & Vector-valued tropical rational map / CPWL mapping & - \\
$\omega$ & Standard generic symbol for a maximal linear region (Definition) & $\mathbb{R}^d$ \\

\midrule
\multicolumn{3}{l}{\textit{\textbf{5. Computational Geometry, Hyperplane Arrangements, and Voronoi Notations}}} \\
\midrule
$S, \hat{S}$ & Generic sets of unweighted and weighted geometric sites & $\mathbb{R}^d, \mathbb{R}^d \times \mathbb{R}$ \\
$c_j$ & A generic geometric generator point (site) & $\mathbb{R}^d$ \\
$V_j, P_j$ & The $j$-th cell in Standard and Power Voronoi Diagrams & $\mathbb{R}^d$ \\
$d_{\text{pow}}(x, c_j)$ & Power distance (Laguerre distance) defined as $\|x - c_j\|^2 - w_j$ & $\mathbb{R}$ \\
$\mathcal{V}(K)$ & Power Voronoi Diagram generated by key set $K$ & - \\
$C_j, \text{int}(C_j)$ & The $j$-th Power Voronoi cell and its interior & $\mathbb{R}^d$ \\
$w_j$ & Intrinsic weight of the $j$-th site, defined as $\|k_j\|^2$ & $\mathbb{R}_{\ge 0}$ \\
$R(n, d)$ & Number of regions induced by $n$ hyperplanes in $d$-dimensional space. Defined as: $R(n, d) \coloneqq \sum_{j=0}^{d} \binom{n}{j}$ & $\mathbb{Z}^+$ \\
$f_0(P)$ & Vertex counting function (number of 0-dimensional faces) & $\mathbb{Z}^+$ \\
$A_{\Omega}, b_{\Omega}$ & Local affine mapping parameters within region $\Omega$ & $\mathbb{R}^{d \times d}, \mathbb{R}^d$ \\

\midrule
\multicolumn{3}{l}{\textit{\textbf{6. Proof-Specific Constants and Auxiliaries (Appx B - G)}}} \\
\midrule
$\mathcal{P}_{\text{MHSA}}, \mathcal{P}_{\text{FFN}}$ & Space partitions induced by MHSA and FFN respectively & - \\
$\mathcal{M}_{\text{MHSA}}, \mathcal{M}_{\text{FFN}}$ & Region counts in MHSA and FFN sub-layers & $\mathbb{Z}^+$ \\
$C_d$ & Complexity constant for Minkowski sums in $d$ dimensions & $\mathbb{R}_{>0}$ \\
$\mathcal{S}_l, \mathcal{R}_l$ & Induced total partition and set of regions after $l$ layers & - \\
$N_l$ & Cumulative number of linear regions after $l$ layers & $\mathbb{Z}^+$ \\
$\Omega$ & A specific maximal linear region in a partition $\mathcal{R}_l$ (used in proofs) & $\mathbb{R}^d$ \\
$\text{Refine}(\Omega)$ & Preimage intersection operator for recursive partitioning & - \\
$p_j$ (scalar) & Parameter for 1D parabolic lifting: $(j - 0.5)/N$ & $[0, 1]$ \\
$w, s(x)$ & Sawtooth teeth count and the constructive sawtooth function & $\mathbb{Z}^+, [0, 1] \to [0, 1]$ \\
$\Omega_0, J(x)$ & Canonical input domain and the mapping's Jacobian matrix & $[0, 1]^d, \mathbb{R}^{d \times d}$ \\
$\sigma_i, \beta$ & Fold sign in $\{-1, 1\}$; Hessian spectral bound constant & - \\
\end{longtable}

\begin{table}[htbp]
    \centering
    \small 
    \caption{Rigorous mapping between Standard DL Domain, Tropical Domain, and Geometry.}
    \label{tab:algebraic_mapping}
    \renewcommand{\arraystretch}{1.3}
    \setlength{\tabcolsep}{3pt} 
    \begin{tabular}{p{0.14\linewidth} p{0.23\linewidth} p{0.23\linewidth} p{0.31\linewidth}}
        \toprule
        \textbf{Concept} & \textbf{Standard DL Domain}   & \textbf{Tropical Domain ($\mathbb{T}$)} & \textbf{Geometric Interpretation} \\
        \midrule
        Addition         & LogSumExp ($A \oplus_\tau B$) & Max ($A \oplus B$)                      & Convex Hull / Supremum \\
        Multiplication   & Standard Add. ($A+B$)         & Tropical Prod. ($A \otimes B$)          & Minkowski Sum ($P_A + P_B$) \\
        Division         & Standard Sub. ($A-B$)         & Tropical Div. ($A \oslash B$)           & Diff. of Convex (DC) / Subtraction \\
        Residual         & $x + \text{Layer}(x)$         & $x \otimes \text{Layer}(x)$             & Spatial Translation \& Degree Shift \\
        Attn. Denom.     & Softmax Denominator           & Tropical Polynomial                     & Convex Polyhedral Surface \\
        MHSA Concat      & Concat. \& Projection         & Tropical Prod. $\bigotimes_{h} f_h$     & Common Refinement of Normal Fans \\
        \bottomrule
    \end{tabular}
\end{table}

\section{Proof of Theorem \ref{thm:voronoi}}
\label{app:proof_theorem_voronoi}

\begin{proof}
We aim to establish the set equality $\mathcal{R}_j^{\text{Attn}} = \mathcal{C}_j^{\text{Power}}$ by showing that the defining condition of the Attention Region is algebraically equivalent to the defining condition of the Power Voronoi Cell under the specific weight assignment $w_l = \|k_l\|^2$.

Consider an arbitrary query vector $q \in \mathbb{R}^d$. By definition, $q \in \mathcal{R}_j^{\text{Attn}}$ if and only if the dot product with key $k_j$ is maximal:
\begin{equation}
    \label{eq:condition_start}
    \langle q, k_j \rangle \ge \langle q, k_l \rangle, \quad \forall l \neq j.
\end{equation}

We utilize the polarization identity to express the inner product in terms of Euclidean norms. For any vector $x, y \in \mathbb{R}^d$, the squared Euclidean distance is expanded as:
\begin{equation*}
    \|x - y\|^2 = \|x\|^2 + \|y\|^2 - 2\langle x, y \rangle.
\end{equation*}
Solving for the inner product term $\langle x, y \rangle$:
\begin{equation}
    \label{eq:polarization}
    \langle x, y \rangle = \frac{1}{2} \left( \|x\|^2 + \|y\|^2 - \|x - y\|^2 \right).
\end{equation}

Substituting the identity Eq. \eqref{eq:polarization} into the inequality Eq. \eqref{eq:condition_start} for both the left-hand side ($x=q, y=k_j$) and the right-hand side ($x=q, y=k_l$):
\begin{equation*}
    \frac{1}{2} \left( \|q\|^2 + \|k_j\|^2 - \|q - k_j\|^2 \right) \ge \frac{1}{2} \left( \|q\|^2 + \|k_l\|^2 - \|q - k_l\|^2 \right).
\end{equation*}

Since the scalar factor $\frac{1}{2}$ is positive, we multiply both sides by $2$:
\begin{equation*}
    \|q\|^2 + \|k_j\|^2 - \|q - k_j\|^2 \ge \|q\|^2 + \|k_l\|^2 - \|q - k_l\|^2.
\end{equation*}

Subtracting the common term $\|q\|^2$ (which is non-negative and independent of the index $l$) from both sides:
\begin{equation*}
    \|k_j\|^2 - \|q - k_j\|^2 \ge \|k_l\|^2 - \|q - k_l\|^2.
\end{equation*}

We rearrange the terms to isolate the distance metrics. Subtracting $\|k_j\|^2$ and $\|k_l\|^2$ from their respective sides and multiplying the entire inequality by $-1$ (which reverses the inequality direction from $\ge$ to $\le$):
\begin{align*}
    -\|q - k_j\|^2 + \|k_j\|^2 &\ge -\|q - k_l\|^2 + \|k_l\|^2 \\
    \iff \quad \|q - k_j\|^2 - \|k_j\|^2 &\le \|q - k_l\|^2 - \|k_l\|^2.
\end{align*}

Let us define the weight for the $m$-th site as $w_m := \|k_m\|^2$. The inequality becomes:
\begin{equation*}
    \|q - k_j\|^2 - w_j \le \|q - k_l\|^2 - w_l.
\end{equation*}

This is precisely the defining condition for the Power Voronoi Cell $\mathcal{C}_j^{\text{Power}}$ centered at site $s_j = k_j$ with weight $w_j = \|k_j\|^2$.
Since the derivation consists of a chain of "if and only if" algebraic manipulations, we conclude:
\begin{equation*}
    q \in \mathcal{R}_j^{\text{Attn}} \iff q \in \mathcal{C}_j^{\text{Power}}.
\end{equation*}
Thus, $\mathcal{R}_j^{\text{Attn}} = \mathcal{C}_j^{\text{Power}}$. This completes the proof.
\end{proof}

\section{Proof of Corollary \ref{corr:value_aggregation}}
\label{app:proof_corr_value_aggregation}

\begin{proof}
Let $q$ be a query vector strictly inside the interior of the Voronoi cell $C_j$, denoted as $q \in \text{int}(C_j)$. By the definition of the cell in \textcolor{blue}{Theorem \ref{thm:voronoi}} (and the equivalence established therein), strict interiority implies strict inequality in the attention scores:
\begin{equation*}
    \langle q, k_j \rangle > \langle q, k_l \rangle, \quad \forall l \neq j.
\end{equation*}
Let $S_l := \langle q, k_l \rangle$ denote the score for the $l$-th key. The condition implies $S_j - S_l > 0$ for all $l \neq j$.

Consider the output of the attention mechanism at a finite temperature $\tau > 0$:
\begin{equation*}
    Z(q; \tau) = \frac{\sum_{l=1}^N e^{S_l/\tau} v_l}{\sum_{m=1}^N e^{S_m/\tau}}.
\end{equation*}
To evaluate the limit $\tau \to 0^+$, we multiply the numerator and the denominator by $e^{-S_j/\tau}$:
\begin{equation*}
    Z(q; \tau) = \frac{e^{-S_j/\tau} \sum_{l=1}^N e^{S_l/\tau} v_l}{e^{-S_j/\tau} \sum_{m=1}^N e^{S_m/\tau}} 
    = \frac{v_j + \sum_{l \neq j} e^{(S_l - S_j)/\tau} v_l}{1 + \sum_{m \neq j} e^{(S_m - S_j)/\tau}}.
\end{equation*}
Since $S_l - S_j < 0$ for all $l \neq j$, the exponent terms behave as:
\begin{equation*}
    \lim_{\tau \to 0^+} \frac{S_l - S_j}{\tau} = -\infty \implies \lim_{\tau \to 0^+} e^{(S_l - S_j)/\tau} = 0.
\end{equation*}
Applying this limit to the expression for $Z(q; \tau)$:
\begin{equation*}
    Z(q) := \lim_{\tau \to 0^+} Z(q; \tau) = \frac{v_j + 0}{1 + 0} = v_j.
\end{equation*}
This holds uniformly for all $q \in \text{int}(C_j)$, confirming that $Z(q)$ is a piecewise constant vector field parameterized by the value matrix $V$. This completes the proof.
\end{proof}

\section{Proof of Theorem \ref{thm:newton_polytope}}
\label{app:proof_theorem_newton_polytope}

\begin{proof}
Let $d:=d_{\mathrm{model}}$. We first verify the single-head Newton-polytope representation. Let $q,k_1,\dots,k_N\in\mathbb{R}^{d_k}$ and define
\begin{equation*}
    f_{\mathrm{single}}(q):=\max_{1\leq j\leq N}\langle q,k_j\rangle.
\end{equation*}
Under the generalized max-plus convention adopted in this work, $k_1,\dots,k_N$ are the slope vectors of $f_{\mathrm{single}}$. Therefore,
\begin{equation*}
    \operatorname{Newt}(f_{\mathrm{single}})=\operatorname{conv}\{k_1,\dots,k_N\}.
\end{equation*}
Every vertex of the convex hull of a finite set belongs to that generating set. Hence,
\begin{equation*}
    \left|\operatorname{Vert}\!\left(\operatorname{Newt}(f_{\mathrm{single}})\right)\right|\leq N.
\end{equation*}

We next express the headwise routing potentials in the common pre-head representation space. Following the row-vector convention used in the main text, let $x\in\mathbb{R}^{d}$, let $W_Q^{(h)}\in\mathbb{R}^{d\times d_k}$, and let $k_j^{(h)}\in\mathbb{R}^{d_k}$ be the $j$-th projected key of head $h$. The projected query is $xW_Q^{(h)}\in\mathbb{R}^{d_k}$, and
\begin{equation*}
    f_h(x):=\max_{1\leq j\leq N}\left\langle xW_Q^{(h)},k_j^{(h)}\right\rangle.
\end{equation*}
For row vectors,
\begin{equation*}
    \left\langle xW_Q^{(h)},k_j^{(h)}\right\rangle=xW_Q^{(h)}(k_j^{(h)})^\top=\left\langle x,k_j^{(h)}(W_Q^{(h)})^\top\right\rangle.
\end{equation*}
Define
\begin{equation*}
    a_j^{(h)}:=k_j^{(h)}(W_Q^{(h)})^\top\in\mathbb{R}^{d},\quad A_h:=\{a_1^{(h)},\dots,a_N^{(h)}\}.
\end{equation*}
Then
\begin{equation*}
    f_h(x)=\max_{1\leq j\leq N}\langle x,a_j^{(h)}\rangle
\end{equation*}
and
\begin{equation*}
    P_h=\operatorname{Newt}(f_h)=\operatorname{conv}(A_h).
\end{equation*}
Consequently,
\begin{equation*}
    n_h=|\operatorname{Vert}(P_h)|\leq|A_h|\leq N.
\end{equation*}

By Definition~\ref{def:mhsa_abstraction}, the joint routing potential is the tropical product
\begin{equation*}
    \Psi(x):=\bigotimes_{h=1}^H f_h(x).
\end{equation*}
Multiplication in the max-plus semiring is ordinary addition. Thus,
\begin{equation*}
    \Psi(x)=\sum_{h=1}^H f_h(x)=\sum_{h=1}^H\max_{1\leq j_h\leq N}\langle x,a_{j_h}^{(h)}\rangle.
\end{equation*}
For every $(j_1,\dots,j_H)\in\{1,\dots,N\}^H$,
\begin{equation*}
    \sum_{h=1}^H\langle x,a_{j_h}^{(h)}\rangle\leq\sum_{h=1}^H\max_{1\leq j\leq N}\langle x,a_j^{(h)}\rangle.
\end{equation*}
Taking the maximum over all tuples yields
\begin{equation*}
    \max_{(j_1,\dots,j_H)\in\{1,\dots,N\}^H}\sum_{h=1}^H\langle x,a_{j_h}^{(h)}\rangle\leq\Psi(x).
\end{equation*}
Conversely, for each $h$, choose
\begin{equation*}
    j_h^*\in\operatorname*{arg\,max}_{1\leq j\leq N}\langle x,a_j^{(h)}\rangle.
\end{equation*}
Then
\begin{equation*}
    \Psi(x)=\sum_{h=1}^H\langle x,a_{j_h^*}^{(h)}\rangle\leq\max_{(j_1,\dots,j_H)\in\{1,\dots,N\}^H}\sum_{h=1}^H\langle x,a_{j_h}^{(h)}\rangle.
\end{equation*}
Therefore,
\begin{equation*}
    \Psi(x)=\max_{(j_1,\dots,j_H)\in\{1,\dots,N\}^H}\sum_{h=1}^H\langle x,a_{j_h}^{(h)}\rangle.
\end{equation*}
By linearity of the inner product,
\begin{equation*}
    \Psi(x)=\max_{(j_1,\dots,j_H)\in\{1,\dots,N\}^H}\left\langle x,\sum_{h=1}^H a_{j_h}^{(h)}\right\rangle.
\end{equation*}

Define the finite Minkowski sum
\begin{equation*}
    A_1+\dots+A_H:=\left\{\sum_{h=1}^H a_{j_h}^{(h)}:(j_1,\dots,j_H)\in\{1,\dots,N\}^H\right\}.
\end{equation*}
The preceding max-affine representation gives
\begin{equation*}
    \operatorname{Newt}(\Psi)=\operatorname{conv}(A_1+\dots+A_H).
\end{equation*}
We now prove
\begin{equation*}
    \operatorname{conv}(A_1+\dots+A_H)=P_1+\dots+P_H.
\end{equation*}

Since $A_h\subseteq P_h$ for every $h$,
\begin{equation*}
    A_1+\dots+A_H\subseteq P_1+\dots+P_H.
\end{equation*}
The right-hand side is convex, so
\begin{equation*}
    \operatorname{conv}(A_1+\dots+A_H)\subseteq P_1+\dots+P_H.
\end{equation*}

For the reverse inclusion, let $y\in P_1+\dots+P_H$. Then there exist $y_h\in P_h$ such that
\begin{equation*}
    y=\sum_{h=1}^H y_h.
\end{equation*}
Because $P_h=\operatorname{conv}(A_h)$, there exist coefficients $\lambda_{h,j}\geq0$ satisfying
\begin{equation*}
    \sum_{j=1}^N\lambda_{h,j}=1,\quad y_h=\sum_{j=1}^N\lambda_{h,j}a_j^{(h)}.
\end{equation*}
For every tuple $(j_1,\dots,j_H)$, define
\begin{equation*}
    \lambda_{j_1,\dots,j_H}:=\prod_{h=1}^H\lambda_{h,j_h}.
\end{equation*}
Then $\lambda_{j_1,\dots,j_H}\geq0$ and
\begin{equation*}
    \sum_{j_1=1}^N\cdots\sum_{j_H=1}^N\lambda_{j_1,\dots,j_H}
    =\prod_{h=1}^H\left(\sum_{j_h=1}^N\lambda_{h,j_h}\right)=1.
\end{equation*}
Moreover,
\begin{align*}
    &\sum_{j_1=1}^N\cdots\sum_{j_H=1}^N\lambda_{j_1,\dots,j_H}\left(\sum_{h=1}^H a_{j_h}^{(h)}\right)=\sum_{h=1}^H\sum_{j_1=1}^N\cdots\sum_{j_H=1}^N\left(\prod_{g=1}^H\lambda_{g,j_g}\right)a_{j_h}^{(h)}=\sum_{h=1}^H\left(\sum_{j_h=1}^N\lambda_{h,j_h}a_{j_h}^{(h)}\right)
    \prod_{\substack{g=1\\g\neq h}}^H\left(\sum_{j_g=1}^N\lambda_{g,j_g}\right)\\
    &\quad=\sum_{h=1}^H\left(\sum_{j_h=1}^N\lambda_{h,j_h}a_{j_h}^{(h)}\right)
    =\sum_{h=1}^H y_h=y.
\end{align*}
Thus, $y$ is a convex combination of points in $A_1+\dots+A_H$, and
\begin{equation*}
    P_1+\dots+P_H\subseteq\operatorname{conv}(A_1+\dots+A_H).
\end{equation*}
Combining both inclusions gives
\begin{equation*}
    P_{\Sigma}=\operatorname{Newt}(\Psi)=P_1+\dots+P_H.
\end{equation*}

We next prove the universal product bound. Let
\begin{equation*}
    \mathcal{V}_h:=\operatorname{Vert}(P_h).
\end{equation*}
Since $P_h=\operatorname{conv}(\mathcal{V}_h)$, the identity just proved gives
\begin{equation*}
    P_{\Sigma}=\operatorname{conv}(\mathcal{V}_1+\dots+\mathcal{V}_H).
\end{equation*}
Every vertex of this convex hull belongs to its generating set. Hence,
\begin{equation*}
    V_{\mathrm{multi}}\leq|\mathcal{V}_1+\dots+\mathcal{V}_H|.
\end{equation*}
The map
\begin{equation*}
    \mathcal{V}_1\times\dots\times\mathcal{V}_H\longrightarrow\mathcal{V}_1+\dots+\mathcal{V}_H,\quad (v_1,\dots,v_H)\longmapsto v_1+\dots+v_H,
\end{equation*}
is surjective. Therefore,
\begin{equation*}
    |\mathcal{V}_1+\dots+\mathcal{V}_H|\leq|\mathcal{V}_1\times\dots\times\mathcal{V}_H|=\prod_{h=1}^H|\mathcal{V}_h|=\prod_{h=1}^H n_h.
\end{equation*}
Since $n_h\leq N$,
\begin{equation*}
    V_{\mathrm{multi}}\leq\prod_{h=1}^H n_h\leq N^H.
\end{equation*}

We now prove the intrinsic-dimension bound. If $d_{\Sigma}=0$, then $P_{\Sigma}$ is a single point and
\begin{equation*}
    V_{\mathrm{multi}}=1.
\end{equation*}
If $d_{\Sigma}=1$, then $P_{\Sigma}$ is a nondegenerate line segment and
\begin{equation*}
    V_{\mathrm{multi}}=2.
\end{equation*}
Assume henceforth that
\begin{equation*}
    d_{\Sigma}\geq2,\quad H\geq d_{\Sigma}.
\end{equation*}

For each $h$, choose $p_h\in P_h$ and define
\begin{equation*}
    \overline{P}_h:=P_h-p_h,\quad \overline{P}_{\Sigma}:=P_{\Sigma}-\sum_{h=1}^H p_h.
\end{equation*}
Then
\begin{align*}
    \overline{P}_{\Sigma}
    &=\left(P_1+\dots+P_H\right)-\sum_{h=1}^H p_h=(P_1-p_1)+\dots+(P_H-p_H)=\overline{P}_1+\dots+\overline{P}_H.
\end{align*}
Translations preserve vertices and affine dimension. Thus,
\begin{equation*}
    |\operatorname{Vert}(\overline{P}_h)|=n_h,\quad |\operatorname{Vert}(\overline{P}_{\Sigma})|=V_{\mathrm{multi}},\quad \dim(\operatorname{aff}(\overline{P}_{\Sigma}))=d_{\Sigma}.
\end{equation*}
Because $p_h\in P_h$,
\begin{equation*}
    0\in\overline{P}_h
\end{equation*}
for every $h$. Hence, for every $z\in\overline{P}_h$,
\begin{equation*}
    z=0+\dots+0+z+0+\dots+0\in\overline{P}_1+\dots+\overline{P}_H=\overline{P}_{\Sigma}.
\end{equation*}
Therefore,
\begin{equation*}
    \overline{P}_h\subseteq\overline{P}_{\Sigma}.
\end{equation*}

Since $0\in\overline{P}_{\Sigma}$, define the linear space
\begin{equation*}
    E:=\operatorname{aff}(\overline{P}_{\Sigma})=\operatorname{span}(\overline{P}_{\Sigma}).
\end{equation*}
Then
\begin{equation*}
    \dim(E)=d_{\Sigma},\quad \overline{P}_h\subseteq E.
\end{equation*}

Let $\Delta\subset E$ be a full-dimensional $d_{\Sigma}$-simplex, let $\varepsilon>0$, and define
\begin{equation*}
    Q_h:=\overline{P}_h+\varepsilon\Delta.
\end{equation*}
Because $\varepsilon\Delta$ is full-dimensional in $E$, each $Q_h$ is full-dimensional in $E$. Define
\begin{equation*}
    m_h:=|\operatorname{Vert}(Q_h)|.
\end{equation*}
Since $\Delta$ has $d_{\Sigma}+1$ vertices, the universal product bound gives
\begin{equation*}
    m_h\leq|\operatorname{Vert}(\overline{P}_h)|\,|\operatorname{Vert}(\varepsilon\Delta)|=n_h(d_{\Sigma}+1).
\end{equation*}

Define
\begin{equation*}
    Q_{\Sigma}:=Q_1+\dots+Q_H.
\end{equation*}
Then
\begin{align*}
    Q_{\Sigma} &=\sum_{h=1}^H(\overline{P}_h+\varepsilon\Delta)=\left(\sum_{h=1}^H\overline{P}_h\right)+
    \underbrace{\varepsilon\Delta+\dots+\varepsilon\Delta}_{H\ \mathrm{times}}.
\end{align*}
Since $\Delta$ is convex,
\begin{equation*}
    \underbrace{\Delta+\dots+\Delta}_{H\ \mathrm{times}}=H\Delta.
\end{equation*}
Indeed, for $\delta_1,\dots,\delta_H\in\Delta$,
\begin{equation*}
    \delta_1+\dots+\delta_H=H\left(\frac{1}{H}\sum_{h=1}^H\delta_h\right)\in H\Delta,
\end{equation*}
while every $H\delta\in H\Delta$ satisfies
\begin{equation*}
    H\delta=\delta+\dots+\delta.
\end{equation*}
Therefore,
\begin{equation*}
    Q_{\Sigma}=\overline{P}_{\Sigma}+H\varepsilon\Delta.
\end{equation*}

For a polytope $R\subseteq E$, let $\mathcal{N}_E(R)$ denote its normal fan relative to $E$. The normal fan of a Minkowski sum is the common refinement of the normal fans of its summands. Hence,
\begin{equation*}
    \mathcal{N}_E(Q_{\Sigma})
    =\mathcal{N}_E(\overline{P}_{\Sigma})\wedge\mathcal{N}_E(H\varepsilon\Delta)
    =\mathcal{N}_E(\overline{P}_{\Sigma})\wedge\mathcal{N}_E(\Delta).
\end{equation*}
Thus, $\mathcal{N}_E(Q_{\Sigma})$ refines $\mathcal{N}_E(\overline{P}_{\Sigma})$. Every full-dimensional cone of $\mathcal{N}_E(\overline{P}_{\Sigma})$ contains at least one full-dimensional cone of $\mathcal{N}_E(Q_{\Sigma})$. Since full-dimensional normal cones correspond bijectively to vertices,
\begin{equation*}
    V_{\mathrm{multi}}=|\operatorname{Vert}(\overline{P}_{\Sigma})|\leq|\operatorname{Vert}(Q_{\Sigma})|.
\end{equation*}

After identifying the $d_{\Sigma}$-dimensional Euclidean space $E$ with $\mathbb{R}^{d_{\Sigma}}$, the polytopes $Q_1,\dots,Q_H$ are full-dimensional $d_{\Sigma}$-polytopes. Since $H\geq d_{\Sigma}$, Corollary~2 of \cite{weibel2012maximal} gives
\begin{equation*}
    |\operatorname{Vert}(Q_{\Sigma})|
    \leq\sum_{\substack{S\subseteq\{1,\dots,H\}\\|S|=d_{\Sigma}-1}}
    \left|\operatorname{Vert}\!\left(\sum_{h\in S}Q_h\right)\right|.
\end{equation*}
For every $S\subseteq\{1,\dots,H\}$ with $|S|=d_{\Sigma}-1$, the universal product bound gives
\begin{equation*}
    \left|\operatorname{Vert}\!\left(\sum_{h\in S}Q_h\right)\right|\leq\prod_{h\in S}m_h.
\end{equation*}
Therefore,
\begin{equation*}
    |\operatorname{Vert}(Q_{\Sigma})|
    \leq\sum_{\substack{S\subseteq\{1,\dots,H\}\\|S|=d_{\Sigma}-1}}\prod_{h\in S}m_h.
\end{equation*}

Define
\begin{equation*}
    m_{\mathrm{tot}}:=\sum_{h=1}^H m_h.
\end{equation*}
Let
\begin{equation*}
    \mathcal{U}:=\bigsqcup_{h=1}^H\operatorname{Vert}(Q_h)
\end{equation*}
be the disjoint union of the summand vertex sets. Then
\begin{equation*}
    |\mathcal{U}|=m_{\mathrm{tot}}.
\end{equation*}
Each term $\prod_{h\in S}m_h$ counts the subsets of $\mathcal{U}$ containing exactly one vertex from each summand indexed by $S$. Hence,
\begin{equation*}
    \sum_{\substack{S\subseteq\{1,\dots,H\}\\|S|=d_{\Sigma}-1}}\prod_{h\in S}m_h
\end{equation*}
counts only those $(d_{\Sigma}-1)$-element subsets of $\mathcal{U}$ whose elements come from distinct summands. These form a subset of all $(d_{\Sigma}-1)$-element subsets of $\mathcal{U}$. Consequently,
\begin{equation*}
    \sum_{\substack{S\subseteq\{1,\dots,H\}\\|S|=d_{\Sigma}-1}}\prod_{h\in S}m_h
    \leq\binom{m_{\mathrm{tot}}}{d_{\Sigma}-1}.
\end{equation*}
Combining the preceding inequalities yields
\begin{equation*}
    V_{\mathrm{multi}}\leq|\operatorname{Vert}(Q_{\Sigma})|\leq\binom{m_{\mathrm{tot}}}{d_{\Sigma}-1}.
\end{equation*}

Since
\begin{equation*}
    \binom{m_{\mathrm{tot}}}{d_{\Sigma}-1}
    =\frac{m_{\mathrm{tot}}(m_{\mathrm{tot}}-1)\cdots(m_{\mathrm{tot}}-d_{\Sigma}+2)}{(d_{\Sigma}-1)!}
    \leq\frac{m_{\mathrm{tot}}^{d_{\Sigma}-1}}{(d_{\Sigma}-1)!},
\end{equation*}
and
\begin{equation*}
    m_{\mathrm{tot}}=\sum_{h=1}^H m_h\leq(d_{\Sigma}+1)\sum_{h=1}^H n_h,
\end{equation*}
we obtain
\begin{align*}
    V_{\mathrm{multi}}
    &\leq\frac{m_{\mathrm{tot}}^{d_{\Sigma}-1}}{(d_{\Sigma}-1)!}
    \leq\frac{(d_{\Sigma}+1)^{d_{\Sigma}-1}}{(d_{\Sigma}-1)!}
    \left(\sum_{h=1}^H n_h\right)^{d_{\Sigma}-1}.
\end{align*}
Define
\begin{equation*}
    C_{d_{\Sigma}}:=\frac{(d_{\Sigma}+1)^{d_{\Sigma}-1}}{(d_{\Sigma}-1)!}.
\end{equation*}
Then $C_{d_{\Sigma}}>0$ depends only on $d_{\Sigma}$, and
\begin{equation*}
    V_{\mathrm{multi}}\leq C_{d_{\Sigma}}\left(\sum_{h=1}^H n_h\right)^{d_{\Sigma}-1}.
\end{equation*}
Since $n_h\leq N$,
\begin{equation*}
    \sum_{h=1}^H n_h\leq\sum_{h=1}^H N=HN,
\end{equation*}
and therefore
\begin{equation*}
    V_{\mathrm{multi}}\leq C_{d_{\Sigma}}(HN)^{d_{\Sigma}-1}
    =\mathcal{O}((HN)^{d_{\Sigma}-1}).
\end{equation*}

It remains to derive the ambient-dimension bounds. If $H<d$, the universal product bound gives
\begin{equation*}
    V_{\mathrm{multi}}\leq N^H=\mathcal{O}(N^H).
\end{equation*}
Suppose that $H\geq d$. Since $d_{\Sigma}\leq d$,
\begin{equation*}
    H\geq d\geq d_{\Sigma}.
\end{equation*}
If $d_{\Sigma}\geq2$, then $HN\geq1$ and
\begin{equation*}
    V_{\mathrm{multi}}\leq C_{d_{\Sigma}}(HN)^{d_{\Sigma}-1}\leq C_{d_{\Sigma}}(HN)^{d-1}.
\end{equation*}
If $d_{\Sigma}=0$ or $d_{\Sigma}=1$, then
\begin{equation*}
    V_{\mathrm{multi}}\leq2\leq2(HN)^{d-1}.
\end{equation*}
Because $d$ is fixed and $d_{\Sigma}\in\{0,\dots,d\}$, define
\begin{equation*}
    \widehat{C}_d:=\max\{2,C_2,\dots,C_d\}.
\end{equation*}
It follows that
\begin{equation*}
    V_{\mathrm{multi}}\leq\widehat{C}_d(HN)^{d-1}
    =\mathcal{O}((HN)^{d-1}),\quad H\geq d.
\end{equation*}
At the boundary $H=d$, the universal product bound gives $V_{\mathrm{multi}}\leq N^d$, whereas the intrinsic-dimension bound has degree at most $d-1$ in $N$. Substituting $d=d_{\mathrm{model}}$ completes the proof.
\end{proof}

\section{Proof of Theorem \ref{thm:linear_regions}}
\label{app:proof_theorem_linear_regions}

\begin{figure}[h]
    \centering
    \includegraphics[width=\textwidth]{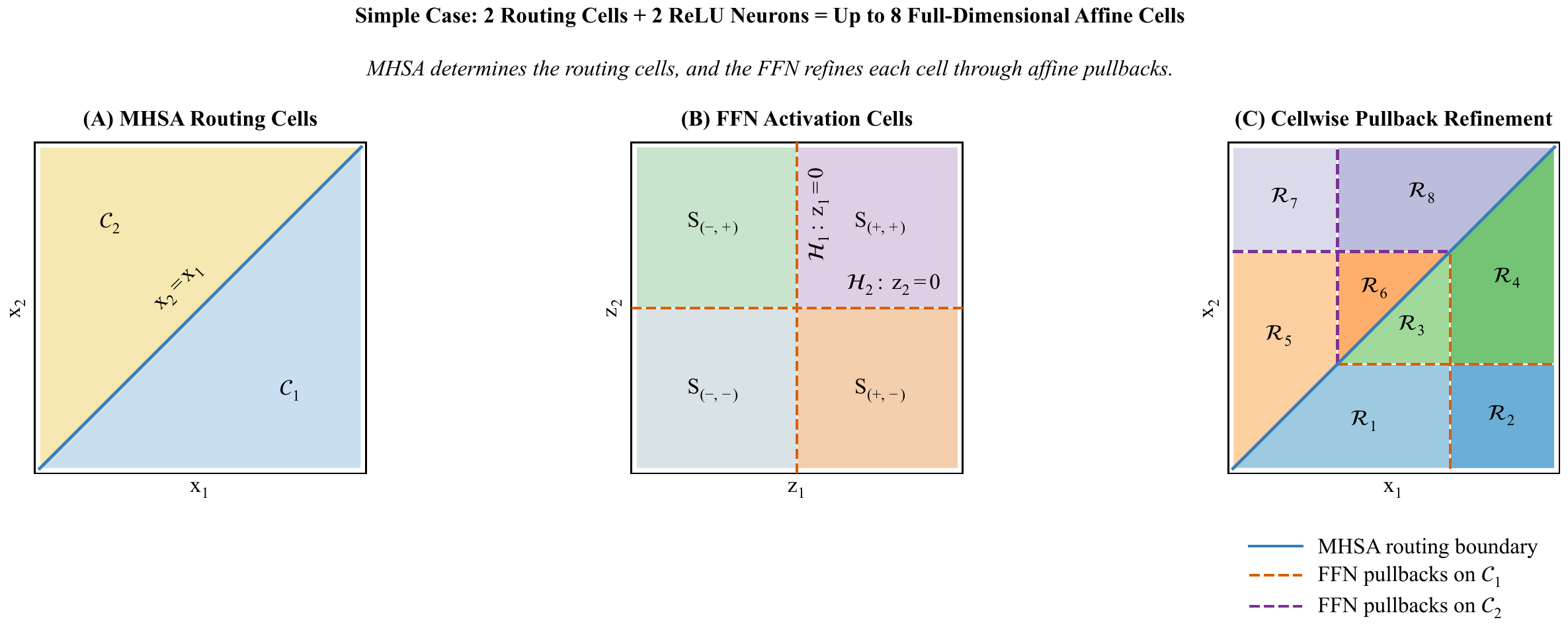}
    \caption{\textbf{Cellwise Pullback Refinement of MHSA Routing Cells by FFN Activation Boundaries.} A two-dimensional conditioned example with two MHSA routing cells and two ReLU neurons. \textbf{(A)} Hard MHSA partitions the query-input space into the routing cells $\mathcal{C}_1$ and $\mathcal{C}_2$. \textbf{(B)} The two activation hyperplanes $\mathcal{H}_1$ and $\mathcal{H}_2$ partition the attention-output space into $\sum_{k=0}^{2}\binom{2}{k}=4$ FFN activation cells. \textbf{(C)} On each routing cell $\mathcal{C}_r$, the affine attention map $T_r$ pulls these activation cells back to sets of the form $\mathcal{C}_r\cap T_r^{-1}(S)$. In the illustrated configuration, the two pulled-back boundaries intersect in the interior of each routing cell and produce four subcells; hence the two routing cells yield $2\times4=8$ full-dimensional affine cells.}
    \label{fig:refinement_illustration}
\end{figure}

\begin{proof}
For a polytope $P$, let
\begin{equation*}
    f_0(P):=|\operatorname{Vert}(P)|,
\end{equation*}
the number of vertices of $P$, equivalently the zeroth entry of its $f$-vector.

Fix a block $\ell\in\{1,\dots,L\}$ and a head $h\in\{1,\dots,H\}$. Let $W_{Q,\ell}^{(h)}\in\mathbb R^{d_{\mathrm{model}}\times d_k}$ and let $k_{\ell,h,j}\in\mathbb R^{d_k}$ be the $j$-th projected key. Define
\begin{equation*}
    a_{\ell,h,j}:=k_{\ell,h,j}(W_{Q,\ell}^{(h)})^\top\in\mathbb R^{d_{\mathrm{model}}}.
\end{equation*}
Under the row-vector convention,
\begin{equation*}
    \langle xW_{Q,\ell}^{(h)},k_{\ell,h,j}\rangle=xW_{Q,\ell}^{(h)}k_{\ell,h,j}^\top=\langle x,k_{\ell,h,j}(W_{Q,\ell}^{(h)})^\top\rangle=\langle x,a_{\ell,h,j}\rangle.
\end{equation*}
Hence the routing potential and its Newton polytope are
\begin{equation*}
    f_{\ell,h}(x):=\max_{1\leq j\leq N}\langle x,a_{\ell,h,j}\rangle,\quad P_{\ell,h}:=\operatorname{Newt}(f_{\ell,h})=\operatorname{conv}\{a_{\ell,h,1},\dots,a_{\ell,h,N}\}.
\end{equation*}
Every vertex of $P_{\ell,h}$ belongs to its generating set, so
\begin{equation*}
    f_0(P_{\ell,h})\leq|\{a_{\ell,h,1},\dots,a_{\ell,h,N}\}|\leq N.
\end{equation*}

For $v\in\operatorname{Vert}(P_{\ell,h})$, define its outer normal cone by
\begin{equation*}
    \mathcal C_{P_{\ell,h}}(v):=\{x\in\mathbb R^{d_{\mathrm{model}}}:\langle x,v\rangle\geq\langle x,u\rangle\text{ for every }u\in P_{\ell,h}\}.
\end{equation*}
The interiors of the full-dimensional cones of the outer normal fan $\operatorname{NF}(P_{\ell,h})$ are precisely the unique-winner routing cells, and these cones correspond bijectively to the vertices of $P_{\ell,h}$ \citep{ziegler2012lectures}. Therefore, the number of full-dimensional routing cells of head $h$ is
\begin{equation*}
    \mathcal M_{\ell,h}=f_0(P_{\ell,h})\leq N.
\end{equation*}

The joint routing partition of block $\ell$ is
\begin{equation*}
    \mathcal P_{\mathrm{att},\ell}:=\bigwedge_{h=1}^{H}\operatorname{NF}(P_{\ell,h}).
\end{equation*}
Since the normal fan of a Minkowski sum equals the common refinement of the normal fans of its summands,
\begin{equation*}
    \mathcal P_{\mathrm{att},\ell}=\operatorname{NF}(P_{\ell,\Sigma}),\quad P_{\ell,\Sigma}:=P_{\ell,1}+\cdots+P_{\ell,H}.
\end{equation*}
Thus the number of full-dimensional joint routing cells is
\begin{equation*}
    \mathcal M_{\mathrm{att},\ell}=f_0(P_{\ell,\Sigma})=V_{\mathrm{multi}}^{(\ell)}\leq V_{\max}.
\end{equation*}

Let $\mathcal C_{\ell,1},\dots,\mathcal C_{\ell,\mathcal M_{\mathrm{att},\ell}}$ denote the interiors of these full-dimensional joint routing cones. Fix $\mathcal C_{\ell,r}$. There exists a tuple $(j_{\ell,1,r},\dots,j_{\ell,H,r})\in\{1,\dots,N\}^H$ such that head $h$ selects index $j_{\ell,h,r}$ throughout $\mathcal C_{\ell,r}$. Let $v_{\ell,h,j}$ be the fixed value vector of head $h$ at index $j$, and let $W_{O,\ell}$ be the output projection. Then
\begin{equation*}
    o_{\ell,h}(x)=v_{\ell,h,j_{\ell,h,r}},\quad x\in\mathcal C_{\ell,r},
\end{equation*}
and consequently
\begin{equation*}
    o_{\ell}^{\mathrm{MHSA}}(x)=\operatorname{Concat}\!\left(v_{\ell,1,j_{\ell,1,r}},\dots,v_{\ell,H,j_{\ell,H,r}}\right)W_{O,\ell}.
\end{equation*}
Therefore, after residual addition and the admissible affine operations, the attention sublayer agrees on $\mathcal C_{\ell,r}$ with an affine map
\begin{equation*}
    T_{\ell,r}(x)=xB_{\ell,r}+c_{\ell,r}.
\end{equation*}

Consider the ReLU FFN
\begin{equation*}
    F_\ell(z):=\sigma(zW_{1,\ell}+b_{1,\ell})W_{2,\ell}+b_{2,\ell},
\end{equation*}
where $W_{1,\ell}\in\mathbb R^{d_{\mathrm{model}}\times d_{\mathrm{ff}}}$, $W_{2,\ell}\in\mathbb R^{d_{\mathrm{ff}}\times d_{\mathrm{model}}}$, and $\sigma(t)=\max\{0,t\}$ is applied coordinatewise. If $w_{\ell,i}$ is the $i$-th column of $W_{1,\ell}$, then the $i$-th proper activation boundary is
\begin{equation*}
    \mathcal H_{\ell,i}:=\{z\in\mathbb R^{d_{\mathrm{model}}}:\langle z,w_{\ell,i}\rangle+(b_{1,\ell})_i=0\}.
\end{equation*}

Let $Z(m,q)$ denote the maximum number of full-dimensional regions induced by $m$ proper affine hyperplanes in $\mathbb R^q$. Its initial conditions are
\begin{equation*}
    Z(0,q)=1,\quad Z(m,0)=1.
\end{equation*}
Deleting the $m$-th hyperplane leaves at most $Z(m-1,q)$ regions, while its restriction is subdivided into at most $Z(m-1,q-1)$ regions. Hence
\begin{equation*}
    Z(m,q)\leq Z(m-1,q)+Z(m-1,q-1).
\end{equation*}
We prove by induction on $m+q$ that
\begin{equation*}
    Z(m,q)\leq\sum_{k=0}^{q}\binom{m}{k},
\end{equation*}
where $\binom{m}{k}:=0$ for $k<0$ or $k>m$. The result is immediate for $m=0$ or $q=0$. For $m,q\geq1$, the induction hypothesis gives
\begin{align*}
    Z(m,q)
    &\leq Z(m-1,q)+Z(m-1,q-1)\\
    &\leq\sum_{k=0}^{q}\binom{m-1}{k}+\sum_{k=0}^{q-1}\binom{m-1}{k}\\
    &=\sum_{k=0}^{q}\binom{m-1}{k}+\sum_{k=1}^{q}\binom{m-1}{k-1}\\
    &=\sum_{k=0}^{q}\left[\binom{m-1}{k}+\binom{m-1}{k-1}\right]\\
    &=\sum_{k=0}^{q}\binom{m}{k}.
\end{align*}
Setting $m=d_{\mathrm{ff}}$ and $q=d_{\mathrm{model}}$ yields the Zaslavsky bound \citep{zaslavsky1975facing}
\begin{equation*}
    Z(d_{\mathrm{ff}},d_{\mathrm{model}})\leq\sum_{k=0}^{d_{\mathrm{model}}}\binom{d_{\mathrm{ff}}}{k}=\sum_{k=0}^{\min\{d_{\mathrm{model}},d_{\mathrm{ff}}\}}\binom{d_{\mathrm{ff}}}{k}=B_{d_{\mathrm{model}},d_{\mathrm{ff}}}.
\end{equation*}

On a fixed FFN activation cell, let $D_{\ell,\eta}\in\mathbb R^{d_{\mathrm{ff}}\times d_{\mathrm{ff}}}$ be its diagonal activation-indicator matrix. Then
\begin{align*}
    F_\ell(z)
    &=(zW_{1,\ell}+b_{1,\ell})D_{\ell,\eta}W_{2,\ell}+b_{2,\ell}\\
    &=zW_{1,\ell}D_{\ell,\eta}W_{2,\ell}+b_{1,\ell}D_{\ell,\eta}W_{2,\ell}+b_{2,\ell},
\end{align*}
so $F_\ell$ is affine on each activation cell.

For $x\in\mathcal C_{\ell,r}$, the pullback of $\mathcal H_{\ell,i}$ through $T_{\ell,r}$ is determined by
\begin{align*}
    \langle T_{\ell,r}(x),w_{\ell,i}\rangle+(b_{1,\ell})_i
    &=\langle xB_{\ell,r}+c_{\ell,r},w_{\ell,i}\rangle+(b_{1,\ell})_i\\
    &=\langle x,B_{\ell,r}w_{\ell,i}\rangle+\langle c_{\ell,r},w_{\ell,i}\rangle+(b_{1,\ell})_i.
\end{align*}
Thus each pullback is empty, equal to $\mathcal C_{\ell,r}$, or the intersection of $\mathcal C_{\ell,r}$ with a proper affine hyperplane. An identically vanishing pullback introduces no subdivision. Hence at most $d_{\mathrm{ff}}$ proper affine hyperplanes subdivide $\mathcal C_{\ell,r}$.

Their arrangement has at most $B_{d_{\mathrm{model}},d_{\mathrm{ff}}}$ full-dimensional chambers. Since both the arrangement chambers and $\mathcal C_{\ell,r}$ are convex, intersecting them with $\mathcal C_{\ell,r}$ cannot increase their number. Therefore,
\begin{equation*}
    \#\{\text{full-dimensional affine cells inside }\mathcal C_{\ell,r}\}\leq B_{d_{\mathrm{model}},d_{\mathrm{ff}}}.
\end{equation*}
Equivalently, if $S$ is an FFN activation cell, the corresponding complete-block cells are
\begin{equation*}
    \mathcal C_{\ell,r}\cap T_{\ell,r}^{-1}(S).
\end{equation*}
This cellwise pullback refinement is illustrated in Figure~\ref{fig:refinement_illustration}.

Summing over the joint routing cells gives
\begin{align*}
    \mathcal M_{\mathrm{block},\ell}
    \leq\sum_{r=1}^{\mathcal M_{\mathrm{att},\ell}}B_{d_{\mathrm{model}},d_{\mathrm{ff}}} =\mathcal M_{\mathrm{att},\ell}B_{d_{\mathrm{model}},d_{\mathrm{ff}}} =V_{\mathrm{multi}}^{(\ell)}B_{d_{\mathrm{model}},d_{\mathrm{ff}}}\leq V_{\max}B_{d_{\mathrm{model}},d_{\mathrm{ff}}}.
\end{align*}
Define
\begin{equation*}
    M:=V_{\max}B_{d_{\mathrm{model}},d_{\mathrm{ff}}}.
\end{equation*}
Thus every block has at most $M$ full-dimensional convex polyhedral affine cells.

Let $\Phi_\ell$ denote the composition of the first $\ell$ blocks, with $\Phi_0:=\operatorname{id}_{\mathbb R^{d_{\mathrm{model}}}}$. We construct a collection $\mathscr Q_\ell$ of full-dimensional convex polyhedral cells with pairwise disjoint interiors such that $\Phi_\ell$ is affine on every $Q\in\mathscr Q_\ell$. Set
\begin{equation*}
    \mathscr Q_0:=\{\mathbb R^{d_{\mathrm{model}}}\},\quad \mu_0:=|\mathscr Q_0|=1.
\end{equation*}

Assume $\mathscr Q_\ell$ has been constructed. For every $Q\in\mathscr Q_\ell$, write
\begin{equation*}
    \Phi_\ell(x)=xA_Q+b_Q,\quad x\in Q.
\end{equation*}
Let $\mathscr S_{\ell+1}$ be the collection of full-dimensional convex affine cells of block $\ell+1$, so that
\begin{equation*}
    |\mathscr S_{\ell+1}|\leq M.
\end{equation*}
For $Q\in\mathscr Q_\ell$ and $S\in\mathscr S_{\ell+1}$, define
\begin{equation*}
    Q_S:=Q\cap(\Phi_\ell|_Q)^{-1}(S).
\end{equation*}
Because $\Phi_\ell|_Q$ is affine and $Q$ and $S$ are convex, $Q_S$ is convex. Retain the sets with $\dim(Q_S)=d_{\mathrm{model}}$ and define
\begin{equation*}
    \mathscr Q_{\ell+1}:=\{Q_S:Q\in\mathscr Q_\ell,\ S\in\mathscr S_{\ell+1},\ \dim(Q_S)=d_{\mathrm{model}}\}.
\end{equation*}

If block $\ell+1$ agrees on $S$ with $G_{\ell+1,S}(z)=zC_S+d_S$, then for every $x\in Q_S$,
\begin{align*}
    \Phi_{\ell+1}(x)=G_{\ell+1,S}(\Phi_\ell(x))=(xA_Q+b_Q)C_S+d_S=x(A_QC_S)+(b_QC_S+d_S).
\end{align*}
Hence $\Phi_{\ell+1}$ is affine on every cell of $\mathscr Q_{\ell+1}$.

By Assumption~\ref{ass:cellwise_noncollapse}, the pullback to $Q$ of every proper boundary of block $\ell+1$ is either empty or lower-dimensional in $Q$. Therefore,
\begin{equation*}
    Q\setminus\bigcup_{S\in\mathscr S_{\ell+1}}Q_S
\end{equation*}
is lower-dimensional. Thus $\mathscr Q_{\ell+1}$ covers the full-dimensional part of the domain up to lower-dimensional boundaries.

For each fixed $Q\in\mathscr Q_\ell$, every $S\in\mathscr S_{\ell+1}$ produces at most one convex cell $Q_S$. Hence
\begin{align*}
    \mu_{\ell+1}:=|\mathscr Q_{\ell+1}|
    \leq\sum_{Q\in\mathscr Q_\ell}|\mathscr S_{\ell+1}|\leq\sum_{Q\in\mathscr Q_\ell}M=M|\mathscr Q_\ell|=M\mu_\ell.
\end{align*}
Since $\mu_0=1$,
\begin{equation*}
    \mu_1\leq M,\quad \mu_2\leq M\mu_1\leq M^2.
\end{equation*}
If $\mu_\ell\leq M^\ell$, then
\begin{equation*}
    \mu_{\ell+1}\leq M\mu_\ell\leq M^{\ell+1}.
\end{equation*}
Therefore, by induction,
\begin{equation*}
    \mu_L\leq M^L=\left(V_{\max}B_{d_{\mathrm{model}},d_{\mathrm{ff}}}\right)^L.
\end{equation*}

Let $\nu_L$ denote the number of maximal connected full-dimensional affine regions of $\Phi_L$. The complement of the union of $\mathscr Q_L$ is lower-dimensional, so every maximal full-dimensional affine region contains at least one cell of $\mathscr Q_L$. Distinct maximal affine regions contain disjoint cells. Hence
\begin{equation*}
    \nu_L\leq|\mathscr Q_L|=\mu_L.
\end{equation*}
Consequently,
\begin{equation*}
    \mathcal N\!\left(\mathcal T(L,H,d_{\mathrm{model}},d_{\mathrm{ff}})\right)=\nu_L\leq\left(V_{\max}B_{d_{\mathrm{model}},d_{\mathrm{ff}}}\right)^L.
\end{equation*}

It remains to apply Theorem~\ref{thm:newton_polytope}. If $H<d_{\mathrm{model}}$, then
\begin{equation*}
    V_{\max}\leq N^H,
\end{equation*}
and therefore
\begin{align*}
    \mathcal N\!\left(\mathcal T(L,H,d_{\mathrm{model}},d_{\mathrm{ff}})\right) \leq\left(B_{d_{\mathrm{model}},d_{\mathrm{ff}}}N^H\right)^L=B_{d_{\mathrm{model}},d_{\mathrm{ff}}}^{\,L}N^{HL}=\mathcal O(N^{HL}).
\end{align*}

If $H\geq d_{\mathrm{model}}$, Theorem~\ref{thm:newton_polytope} implies that there exists a constant $C>0$, independent of $N$, such that
\begin{equation*}
    V_{\max}\leq CN^{d_{\mathrm{model}}-1}.
\end{equation*}
Hence
\begin{align*}
    \mathcal N\!\left(\mathcal T(L,H,d_{\mathrm{model}},d_{\mathrm{ff}})\right)
    &\leq\left(CB_{d_{\mathrm{model}},d_{\mathrm{ff}}}N^{d_{\mathrm{model}}-1}\right)^L\\
    &=\left(CB_{d_{\mathrm{model}},d_{\mathrm{ff}}}\right)^LN^{(d_{\mathrm{model}}-1)L}\\
    &=\mathcal O\!\left(N^{(d_{\mathrm{model}}-1)L}\right).
\end{align*}
This proves the theorem.
\end{proof}

\section{Proof of Theorem \ref{thm:lower_bound_linear_regions}}
\label{app:proof_theorem_lower_bound}

The construction follows the full-branch replication principle used in classical lower bounds for deep piecewise-linear networks \citep{montufar2014number,telgarsky2016benefits}. It uses the standard normalization-free residual block
\begin{equation*}
    u=x+\operatorname{MHSA}(x),\quad \mathcal B(x)=u+\operatorname{FFN}(u),
\end{equation*}
without affine preprocessing or data-dependent normalization.

\begin{lemma}[Homogeneous Parabolic Routing]
\label{lem:homogeneous_parabolic_routing}
Let $N\geq1$ and define $p_j:=(j-\frac12)/N$ for $j\in\{1,\dots,N\}$. For $(s,t)\in\mathbb R^2$ with $t>0$, let
\begin{equation*}
    q(s,t):=(s,t),\quad k_j:=\left(p_j,-\frac12p_j^2\right).
\end{equation*}
Then
\begin{equation*}
    \operatorname*{arg\,max}_{1\leq j\leq N}\langle q(s,t),k_j\rangle=\operatorname*{arg\,min}_{1\leq j\leq N}\left|\frac{s}{t}-p_j\right|.
\end{equation*}
In particular, key $j$ is the unique maximizer whenever
\begin{equation*}
    \frac{j-1}{N}<\frac{s}{t}<\frac{j}{N}.
\end{equation*}
\end{lemma}

\begin{proof}
For every $j\in\{1,\dots,N\}$,
\begin{equation*}
    \langle q(s,t),k_j\rangle=sp_j-\frac12tp_j^2=t\left(\frac{s}{t}p_j-\frac12p_j^2\right)=\frac{t}{2}\left(\frac{s^2}{t^2}-\left(\frac{s}{t}-p_j\right)^2\right).
\end{equation*}
Since $t>0$ and $ts^2/(2t^2)$ is independent of $j$,
\begin{equation*}
    \operatorname*{arg\,max}_{1\leq j\leq N}\langle q(s,t),k_j\rangle=\operatorname*{arg\,min}_{1\leq j\leq N}\left(\frac{s}{t}-p_j\right)^2=\operatorname*{arg\,min}_{1\leq j\leq N}\left|\frac{s}{t}-p_j\right|.
\end{equation*}
The sites satisfy $p_1<\cdots<p_N$, and the midpoint between adjacent sites is
\begin{equation*}
    \frac{p_j+p_{j+1}}{2}=\frac12\left(\frac{j-\frac12}{N}+\frac{j+\frac12}{N}\right)=\frac{j}{N}.
\end{equation*}
Fix $r:=s/t\in((j-1)/N,j/N)$. If $i<j$, then $(p_i+p_j)/2\leq(p_{j-1}+p_j)/2=(j-1)/N<r$, so $|r-p_j|<|r-p_i|$. If $i>j$, then $r<j/N=(p_j+p_{j+1})/2\leq(p_j+p_i)/2$, so $|r-p_j|<|r-p_i|$. Hence $p_j$ is strictly closer to $r$ than every other site, and key $j$ is the unique maximizer.
\end{proof}

\begin{proof}[Proof of Theorem~\ref{thm:lower_bound_linear_regions}]
For compactness, define
\begin{equation*}
    d:=d_{\mathrm{model}},\quad \rho:=\min\{H,d-1\},\quad a:=\frac12,\quad \varepsilon:=\frac{1}{16N},\quad \lambda:=\frac{1}{2\varepsilon}=8N.
\end{equation*}
Because $d\geq2$ and $H\geq1$, we have $\rho\geq1$. Moreover, $\lambda=8N>N$.

Use the first $\rho$ attention heads as active heads and set the outputs of all remaining heads to zero. Let $e_1,\dots,e_d$ denote the standard basis vectors of $\mathbb R^d$. For each active head $h\in\{1,\dots,\rho\}$, use two query--key coordinates and choose
\begin{equation*}
    W_Q^{(h)}:=\begin{bmatrix}e_h&e_d&0_{d\times(d_k-2)}\end{bmatrix}\in\mathbb R^{d\times d_k}.
\end{equation*}
Under the row-vector convention,
\begin{equation*}
    q_h(x)=xW_Q^{(h)}=(x_h,x_d,0,\dots,0).
\end{equation*}
For $j\in\{1,\dots,N\}$, define
\begin{equation*}
    p_j:=\frac{j-\frac12}{N},\quad k_j^{(h)}:=\left(p_j,-\frac12p_j^2,0,\dots,0\right)\in\mathbb R^{d_k}.
\end{equation*}
The score of key $j$ in active head $h$ is
\begin{equation*}
    \left\langle q_h(x),k_j^{(h)}\right\rangle=\left\langle(x_h,x_d,0,\dots,0),\left(p_j,-\frac12p_j^2,0,\dots,0\right)\right\rangle=x_hp_j-\frac12x_dp_j^2.
\end{equation*}
By Lemma~\ref{lem:homogeneous_parabolic_routing}, if $x_d>0$ and
\begin{equation*}
    \frac{j-1}{N}<\frac{x_h}{x_d}<\frac{j}{N},
\end{equation*}
then head $h$ uniquely selects key $j$.

For each active head, use one value coordinate and choose
\begin{equation*}
    v_j^{(h)}:=\left(a(1-p_j),0,\dots,0\right).
\end{equation*}
Choose the output projection so that the first value coordinate of active head $h$ is written into residual coordinate $h$, while every other coordinate receives zero from that head. Thus, when active head $h$ selects key $j_h$, its contribution to the MHSA output is
\begin{equation*}
    a(1-p_{j_h})e_h.
\end{equation*}

Define the common target domain
\begin{equation*}
    \Omega_N:=(0,1)^\rho\times(0,1)^{d-\rho-1}\times(a-\varepsilon,a+\varepsilon)\subset(0,1)^d.
\end{equation*}
When $d-\rho-1=0$, the middle Cartesian factor is omitted. For every routing tuple $\mathbf j=(j_1,\dots,j_\rho)\in\{1,\dots,N\}^{\rho}$, define
\begin{equation*}
    D_{\mathbf j}:=\prod_{h=1}^{\rho}(ap_{j_h}-\varepsilon,ap_{j_h}+\varepsilon)\times(0,1)^{d-\rho-1}\times(a-\varepsilon,a+\varepsilon).
\end{equation*}

We first verify that $D_{\mathbf j}\subset\Omega_N$. The smallest possible active-coordinate endpoint satisfies
\begin{equation*}
    ap_1-\varepsilon=\frac{1}{4N}-\frac{1}{16N}=\frac{3}{16N}>0,
\end{equation*}
while the largest satisfies
\begin{equation*}
    ap_N+\varepsilon=\frac12\left(1-\frac{1}{2N}\right)+\frac{1}{16N}=\frac12-\frac{1}{4N}+\frac{1}{16N}=\frac12-\frac{3}{16N}<1.
\end{equation*}
Hence all active-coordinate intervals lie in $(0,1)$, and therefore $D_{\mathbf j}\subset\Omega_N$.

The boxes $D_{\mathbf j}$ are pairwise disjoint. Indeed, if $\mathbf j\neq\mathbf j'$, there exists an active coordinate $h$ such that $j_h\neq j_h'$. The corresponding centers satisfy
\begin{equation*}
    |ap_{j_h}-ap_{j_h'}|=\frac{a|j_h-j_h'|}{N}\geq\frac{a}{N}=\frac{1}{2N}>2\varepsilon=\frac{1}{8N},
\end{equation*}
so their active-coordinate intervals are disjoint.

We next show that every $D_{\mathbf j}$ lies strictly inside the joint routing cell indexed by $\mathbf j$. Fix $x\in D_{\mathbf j}$ and an active head $h$. Write
\begin{equation*}
    x_h=ap_{j_h}+\delta_h,\quad x_d=a+\delta_d,\quad |\delta_h|<\varepsilon,\quad |\delta_d|<\varepsilon.
\end{equation*}
Using $p_{j_h}=(j_h-\frac12)/N$ and $0\leq(j_h-1)/N<1$,
\begin{equation*}
    x_h-\frac{j_h-1}{N}x_d=\frac{a}{2N}+\delta_h-\frac{j_h-1}{N}\delta_d>\frac{a}{2N}-2\varepsilon=\frac{1}{4N}-\frac{1}{8N}=\frac{1}{8N}>0.
\end{equation*}
Similarly, since $0<j_h/N\leq1$,
\begin{equation*}
    \frac{j_h}{N}x_d-x_h=\frac{a}{2N}+\frac{j_h}{N}\delta_d-\delta_h>\frac{a}{2N}-2\varepsilon=\frac{1}{8N}>0.
\end{equation*}
Moreover,
\begin{equation*}
    x_d>a-\varepsilon=\frac12-\frac{1}{16N}>0.
\end{equation*}
Dividing the preceding strict inequalities by $x_d$ gives
\begin{equation*}
    \frac{j_h-1}{N}<\frac{x_h}{x_d}<\frac{j_h}{N}.
\end{equation*}
Lemma~\ref{lem:homogeneous_parabolic_routing} therefore implies that active head $h$ uniquely selects key $j_h$ throughout $D_{\mathbf j}$.

Fix $\mathbf j$ and $x\in D_{\mathbf j}$. The MHSA output is
\begin{equation*}
    \operatorname{MHSA}(x)=\sum_{h=1}^{\rho}a(1-p_{j_h})e_h.
\end{equation*}
After the attention residual connection,
\begin{equation*}
    u=x+\operatorname{MHSA}(x),
\end{equation*}
so
\begin{equation*}
    u_h=x_h+a(1-p_{j_h})\quad\text{for }h=1,\dots,\rho,\quad u_i=x_i\quad\text{for }i=\rho+1,\dots,d.
\end{equation*}
Since $x_h\in(ap_{j_h}-\varepsilon,ap_{j_h}+\varepsilon)$,
\begin{equation*}
    u_h\in(ap_{j_h}-\varepsilon+a(1-p_{j_h}),ap_{j_h}+\varepsilon+a(1-p_{j_h}))=(a-\varepsilon,a+\varepsilon).
\end{equation*}
Thus the attention residual maps every $D_{\mathbf j}$ onto the same intermediate box
\begin{equation*}
    E_N:=(a-\varepsilon,a+\varepsilon)^\rho\times(0,1)^{d-\rho-1}\times(a-\varepsilon,a+\varepsilon).
\end{equation*}

We now construct the FFN. For every active coordinate $h$, introduce the two hidden units
\begin{equation*}
    \eta_{h,+}(u):=\operatorname{ReLU}(u_h),\quad \eta_{h,-}(u):=\operatorname{ReLU}(-u_h).
\end{equation*}
The construction uses $2\rho=2\min\{H,d-1\}\leq d_{\mathrm{ff}}$ hidden units and therefore fits within the available FFN width. Assign every unused hidden unit zero input weight, bias $-1$, and zero output weight.

For $h=1,\dots,\rho$, define
\begin{equation*}
    F_h(u):=(\lambda-1)\bigl(\eta_{h,+}(u)-\eta_{h,-}(u)\bigr)+a(1-\lambda),
\end{equation*}
and set
\begin{equation*}
    F_i(u):=0,\quad i=\rho+1,\dots,d.
\end{equation*}
For every $z\in\mathbb R$,
\begin{equation*}
    \operatorname{ReLU}(z)-\operatorname{ReLU}(-z)=z.
\end{equation*}
Indeed, if $z\geq0$, then the left-hand side equals $z-0=z$, while if $z<0$, it equals $0-(-z)=z$. Hence
\begin{equation*}
    F_h(u)=(\lambda-1)u_h+a(1-\lambda).
\end{equation*}
After the FFN residual connection, for every active coordinate $h$,
\begin{equation*}
    \mathcal B_h(x)=u_h+F_h(u)=u_h+(\lambda-1)u_h+a(1-\lambda)=\lambda u_h+a(1-\lambda).
\end{equation*}
Substituting $u_h=x_h+a(1-p_{j_h})$ gives
\begin{equation*}
    \mathcal B_h(x)=\lambda\bigl(x_h+a(1-p_{j_h})\bigr)+a(1-\lambda)=\lambda x_h+\lambda a-\lambda ap_{j_h}+a-a\lambda=a+\lambda(x_h-ap_{j_h}).
\end{equation*}
For every passive coordinate,
\begin{equation*}
    \mathcal B_i(x)=x_i,\quad i=\rho+1,\dots,d.
\end{equation*}
Therefore, on $D_{\mathbf j}$, the block agrees with the affine map
\begin{equation*}
    \mathcal B_{\mathbf j}(x)=\left(a+\lambda(x_1-ap_{j_1}),\dots,a+\lambda(x_\rho-ap_{j_\rho}),x_{\rho+1},\dots,x_d\right).
\end{equation*}

For an active coordinate $h$, the inclusion $x_h\in(ap_{j_h}-\varepsilon,ap_{j_h}+\varepsilon)$ implies
\begin{equation*}
    \mathcal B_{\mathbf j,h}(x)\in(a-\lambda\varepsilon,a+\lambda\varepsilon)=\left(\frac12-\frac12,\frac12+\frac12\right)=(0,1),
\end{equation*}
because $\lambda\varepsilon=8N/(16N)=1/2$. The passive coordinates, including $x_d$, are unchanged. Consequently,
\begin{equation*}
    \mathcal B_{\mathbf j}(D_{\mathbf j})=\Omega_N.
\end{equation*}
The inverse branch is
\begin{equation*}
    \mathcal B_{\mathbf j}^{-1}(y)=\left(ap_{j_1}+\frac{y_1-a}{\lambda},\dots,ap_{j_\rho}+\frac{y_\rho-a}{\lambda},y_{\rho+1},\dots,y_d\right).
\end{equation*}
For $y_h\in(0,1)$,
\begin{equation*}
    -\frac{1}{2\lambda}<\frac{y_h-a}{\lambda}<\frac{1}{2\lambda},
\end{equation*}
and since $1/(2\lambda)=\varepsilon$, the inverse active coordinate lies in $(ap_{j_h}-\varepsilon,ap_{j_h}+\varepsilon)$. Thus $\mathcal B_{\mathbf j}^{-1}(\Omega_N)=D_{\mathbf j}$, and $\mathcal B_{\mathbf j}:D_{\mathbf j}\to\Omega_N$ is an affine bijection.

Let
\begin{equation*}
    \mathcal J:=\{1,\dots,N\}^{\rho}.
\end{equation*}
For every word $\boldsymbol{\mathbf j}=(\mathbf j^{(1)},\dots,\mathbf j^{(L)})\in\mathcal J^L$, define
\begin{equation*}
    Q_{\boldsymbol{\mathbf j}}:=\mathcal B_{\mathbf j^{(1)}}^{-1}\circ\mathcal B_{\mathbf j^{(2)}}^{-1}\circ\cdots\circ\mathcal B_{\mathbf j^{(L)}}^{-1}(\Omega_N).
\end{equation*}
Each inverse branch is an affine bijection from $\Omega_N$ onto a nonempty open full-dimensional box contained in $\Omega_N$. Hence every $Q_{\boldsymbol{\mathbf j}}$ is a nonempty open full-dimensional box. Moreover,
\begin{equation*}
    \mathcal B^{\circ t}(Q_{\boldsymbol{\mathbf j}})\subset D_{\mathbf j^{(t+1)}}\quad\text{for }t=0,\dots,L-1,\quad \mathcal B^{\circ L}(Q_{\boldsymbol{\mathbf j}})=\Omega_N,
\end{equation*}
where $D_{\mathbf j^{(1)}}$ is understood when $t=0$. Therefore, the restriction of $\mathcal B^{\circ L}$ to $Q_{\boldsymbol{\mathbf j}}$ is the composition of the prescribed affine branches.

The cells corresponding to distinct words are disjoint. Let $\boldsymbol{\mathbf j}\neq\boldsymbol{\mathbf j}'$, and let $t$ be the first position at which $\mathbf j^{(t)}\neq\mathbf j'^{(t)}$. Their prefixes of length $t-1$ coincide, so the same injective affine prefix map is applied to both cells. Their images under that prefix map lie in the disjoint boxes $D_{\mathbf j^{(t)}}$ and $D_{\mathbf j'^{(t)}}$. Hence
\begin{equation*}
    Q_{\boldsymbol{\mathbf j}}\cap Q_{\boldsymbol{\mathbf j}'}=\varnothing.
\end{equation*}
Since $|\mathcal J|=N^\rho$, the number of constructed cells is
\begin{equation*}
    |\mathcal J^L|=|\mathcal J|^L=(N^\rho)^L=N^{\rho L}.
\end{equation*}

It remains to prove that these cells belong to distinct maximal connected full-dimensional affine regions. On a one-block branch, the $h$-th active coordinate has the form
\begin{equation*}
    \mathcal B_{\mathbf j,h}(x)=\lambda x_h+\beta_{j_h},\quad \beta_j:=a-\lambda ap_j.
\end{equation*}
We claim that on a word $\boldsymbol{\mathbf j}$,
\begin{equation*}
    \mathcal B_h^{\circ L}(x)=\lambda^Lx_h+\sum_{t=1}^{L}\lambda^{L-t}\beta_{j_h^{(t)}}.
\end{equation*}
The formula is immediate for $L=1$. If it holds at depth $\ell$, then
\begin{equation*}
    \lambda\left(\lambda^\ell x_h+\sum_{t=1}^{\ell}\lambda^{\ell-t}\beta_{j_h^{(t)}}\right)+\beta_{j_h^{(\ell+1)}}=\lambda^{\ell+1}x_h+\sum_{t=1}^{\ell+1}\lambda^{\ell+1-t}\beta_{j_h^{(t)}},
\end{equation*}
which proves the claim by induction.

Consider two distinct words. There exists an active coordinate $h$ for which the sequences $(j_h^{(1)},\dots,j_h^{(L)})$ and $(j_h'^{(1)},\dots,j_h'^{(L)})$ differ. Let $t_0$ be their first differing position and define $\Delta_t:=j_h^{(t)}-j_h'^{(t)}$. Since $p_j-p_{j'}=(j-j')/N$,
\begin{equation*}
    \beta_j-\beta_{j'}=-\lambda a(p_j-p_{j'})=-\frac{\lambda a}{N}(j-j').
\end{equation*}
The difference between the two affine intercepts in coordinate $h$ is therefore
\begin{equation*}
    -\frac{\lambda a}{N}\sum_{t=1}^{L}\lambda^{L-t}\Delta_t=-\frac{\lambda a}{N}\left(\lambda^{L-t_0}\Delta_{t_0}+\sum_{t=t_0+1}^{L}\lambda^{L-t}\Delta_t\right).
\end{equation*}
Because $\Delta_{t_0}$ is a nonzero integer,
\begin{equation*}
    \left|\lambda^{L-t_0}\Delta_{t_0}\right|\geq\lambda^{L-t_0}.
\end{equation*}
For every later index, $|\Delta_t|\leq N-1$, and hence
\begin{equation*}
    \left|\sum_{t=t_0+1}^{L}\lambda^{L-t}\Delta_t\right|\leq(N-1)\sum_{t=t_0+1}^{L}\lambda^{L-t}=(N-1)\sum_{r=0}^{L-t_0-1}\lambda^r=(N-1)\frac{\lambda^{L-t_0}-1}{\lambda-1}.
\end{equation*}
Since $\lambda=8N$,
\begin{equation*}
    \frac{N-1}{\lambda-1}=\frac{N-1}{8N-1}<1,
\end{equation*}
and therefore
\begin{equation*}
    (N-1)\frac{\lambda^{L-t_0}-1}{\lambda-1}<(N-1)\frac{\lambda^{L-t_0}}{\lambda-1}<\lambda^{L-t_0}.
\end{equation*}
The later terms cannot cancel the leading nonzero term. Thus the two affine intercepts differ, and the restrictions of $\mathcal B^{\circ L}$ to the two cells are distinct affine maps.

Suppose, for contradiction, that two distinct cells $Q_{\boldsymbol{\mathbf j}}$ and $Q_{\boldsymbol{\mathbf j}'}$ belonged to the same connected full-dimensional affine region. Then there would exist one affine map agreeing with $\mathcal B^{\circ L}$ on nonempty open subsets of both cells. Two affine maps that agree on a nonempty open subset of $\mathbb R^d$ have identical linear parts and biases, so the two branch restrictions would be identical, contradicting the preceding calculation. Therefore, the $N^{\rho L}$ cells belong to pairwise distinct maximal connected full-dimensional affine regions.

Consequently,
\begin{equation*}
    \mathcal N\!\left(\mathcal T(L,H,d,d_{\mathrm{ff}})\right)\geq N^{\rho L}.
\end{equation*}
Since $\rho=\min\{H,d-1\}$,
\begin{equation*}
    \mathcal N\!\left(\mathcal T(L,H,d,d_{\mathrm{ff}})\right)=\begin{cases}\Omega(N^{HL}), & H<d,\\ \Omega(N^{(d-1)L}), & H\geq d.\end{cases}
\end{equation*}

Finally, we verify Assumption~\ref{ass:cellwise_noncollapse}. Let $\Phi_\ell:=\mathcal B^{\circ\ell}$, with $\Phi_0=\operatorname{id}$. Every one-block affine restriction has slope $\lambda$ in each active coordinate and slope $1$ in each passive coordinate. By induction, on every full-dimensional affine cell of $\Phi_\ell$,
\begin{equation*}
    \Phi_{\ell,h}(x)=\lambda^\ell x_h+c_{\ell,h},\quad h=1,\dots,\rho,
\end{equation*}
for suitable constants $c_{\ell,h}$.

For distinct keys $i\neq j$ in active head $h$, the routing-tie functional is
\begin{equation*}
    g_{h,ij}(z)=(p_i-p_j)z_h-\frac12(p_i^2-p_j^2)z_d.
\end{equation*}
Because the last coordinate is passive, $\Phi_{\ell,d}(x)=x_d$. Hence its pullback to a prefix cell is
\begin{equation*}
    g_{h,ij}(\Phi_\ell(x))=(p_i-p_j)\lambda^\ell x_h-\frac12(p_i^2-p_j^2)x_d+(p_i-p_j)c_{\ell,h}.
\end{equation*}
Since
\begin{equation*}
    p_i-p_j=\frac{i-j}{N}\neq0,
\end{equation*}
the coefficient $(p_i-p_j)\lambda^\ell$ of $x_h$ is nonzero. Thus no nonzero routing-tie functional vanishes identically on a full-dimensional prefix cell.

For the FFN in block $\ell+1$, restrict a prefix cell further to a full-dimensional routing cell on which active head $h$ selects key $j_h$. The attention residual coordinate is
\begin{equation*}
    u_h(x)=\Phi_{\ell,h}(x)+a(1-p_{j_h})=\lambda^\ell x_h+c_{\ell,h}+a(1-p_{j_h}).
\end{equation*}
The two active ReLU preactivations are $u_h(x)$ and $-u_h(x)$, whose coefficients of $x_h$ are $\lambda^\ell$ and $-\lambda^\ell$, respectively. Both coefficients are nonzero, so neither preactivation vanishes identically on a full-dimensional cell.

Inactive heads have zero query, key, value, and output projections, so their pairwise score differences are the zero functional and are excluded by Assumption~\ref{ass:cellwise_noncollapse}. Every unused ReLU unit has constant preactivation $-1$, whose pullback is also $-1$. Hence the constructed network satisfies Assumption~\ref{ass:cellwise_noncollapse}.

Substituting $d=d_{\mathrm{model}}$ completes the proof.
\end{proof}

\section{Proof of Theorem~\ref{thm:finite_temperature_stability}}
\label{app:proof_theorem_finite_temperature_stability}

We regard the query, key, and value vectors as column vectors. For a differentiable map $F:\mathbb R^{d_k}\to\mathbb R^{d_v}$, its Jacobian $JF(q)\in\mathbb R^{d_v\times d_k}$ is defined by
\begin{equation*}
    F(q+h)=F(q)+JF(q)h+o(\|h\|_2).
\end{equation*}

\begin{lemma}[Spectral Bound for a Weighted Covariance]
\label{lem:weighted_covariance_bound}
Let $p_1,\dots,p_N\geq0$ satisfy $\sum_{j=1}^{N}p_j=1$, and let $\xi_1,\dots,\xi_N\in\mathbb R^d$. Fix $i\in\{1,\dots,N\}$ and assume that $\xi_i=0$ and $\|\xi_j\|_2\leq K$ for every $j\neq i$. Define
\begin{equation*}
    m:=\sum_{j=1}^{N}p_j\xi_j,\quad C:=\sum_{j=1}^{N}p_j\xi_j\xi_j^\top-mm^\top,\quad \ell:=\sum_{j\neq i}p_j.
\end{equation*}
Then $C$ is positive semidefinite and
\begin{equation*}
    \|C\|_2\leq K^2\ell.
\end{equation*}
\end{lemma}

\begin{proof}
Fix $x\in\mathbb R^d$ and define $a_j:=x^\top\xi_j$ and $\bar a:=\sum_{j=1}^{N}p_ja_j=x^\top m$. Then
\begin{equation*}
    x^\top Cx=\sum_{j=1}^{N}p_ja_j^2-\bar a^2.
\end{equation*}
Since $\sum_{j=1}^{N}p_j=1$ and $\sum_{j=1}^{N}p_ja_j=\bar a$,
\begin{equation*}
    \sum_{j=1}^{N}p_j(a_j-\bar a)^2=\sum_{j=1}^{N}p_ja_j^2-2\bar a\sum_{j=1}^{N}p_ja_j+\bar a^2\sum_{j=1}^{N}p_j=\sum_{j=1}^{N}p_ja_j^2-\bar a^2.
\end{equation*}
Therefore,
\begin{equation*}
    x^\top Cx=\sum_{j=1}^{N}p_j(a_j-\bar a)^2\geq0.
\end{equation*}
Thus \(C\succeq0\).

Define
\begin{equation*}
    S:=\sum_{j=1}^{N}p_j\xi_j\xi_j^\top.
\end{equation*}
Since
\begin{equation*}
    S-C=mm^\top\succeq0,
\end{equation*}
we have
\begin{equation*}
    0\preceq C\preceq S.
\end{equation*}
Both matrices are symmetric and positive semidefinite, so
\begin{equation*}
    \|C\|_2=\lambda_{\max}(C)\leq\lambda_{\max}(S)=\|S\|_2.
\end{equation*}
By the triangle inequality for the spectral norm,
\begin{equation*}
    \|S\|_2\leq\sum_{j=1}^{N}p_j\|\xi_j\xi_j^\top\|_2.
\end{equation*}
For any vector \(\xi\),
\begin{equation*}
    \|\xi\xi^\top\|_2=\sup_{\|x\|_2=1}\|\xi\xi^\top x\|_2=\|\xi\|_2\sup_{\|x\|_2=1}|\xi^\top x|=\|\xi\|_2^2.
\end{equation*}
Consequently,
\begin{equation*}
    \|C\|_2\leq\sum_{j=1}^{N}p_j\|\xi_j\|_2^2=\sum_{j\neq i}p_j\|\xi_j\|_2^2\leq K^2\sum_{j\neq i}p_j=K^2\ell.
\end{equation*}
\end{proof}

\begin{proof}[Proof of Theorem~\ref{thm:finite_temperature_stability}]
For compactness, define
\begin{equation*}
    s_j(q):=\langle q,k_j\rangle,\quad Z(q):=\sum_{r=1}^{N}e^{s_r(q)/\tau}.
\end{equation*}
Then
\begin{equation*}
    p_j^{(\tau)}(q)=\frac{e^{s_j(q)/\tau}}{Z(q)},\quad P^{(\tau)}(q)=\tau\log Z(q).
\end{equation*}

We first prove the exact temperature-invariance of top-$1$ routing. For any indices $j,r\in\{1,\dots,N\}$,
\begin{equation*}
    p_j^{(\tau)}(q)\geq p_r^{(\tau)}(q)\iff e^{s_j(q)/\tau}\geq e^{s_r(q)/\tau}\iff s_j(q)\geq s_r(q),
\end{equation*}
because the two softmax probabilities have the same strictly positive denominator, the exponential function is strictly increasing, and \(\tau>0\). Therefore,
\begin{equation*}
    \operatorname*{arg\,max}_{1\leq j\leq N}p_j^{(\tau)}(q)=\operatorname*{arg\,max}_{1\leq j\leq N}s_j(q)=\operatorname*{arg\,max}_{1\leq j\leq N}\langle q,k_j\rangle.
\end{equation*}
Similarly,
\begin{equation*}
    p_j^{(\tau)}(q)=p_r^{(\tau)}(q)\iff s_j(q)=s_r(q)\iff\langle q,k_j-k_r\rangle=0.
\end{equation*}
Hence the pairwise tie sets are identical at finite and zero temperature. When \(k_j\neq k_r\), the common tie set is the hyperplane
\begin{equation*}
    \left\{q\in\mathbb R^{d_k}:\langle q,k_j-k_r\rangle=0\right\}.
\end{equation*}

We next derive an auxiliary concentration bound. Fix \(i\in\{1,\dots,N\}\), \(\delta>0\), and \(q\in C_i^\delta\). By the definition of \(C_i^\delta\),
\begin{equation*}
    s_i(q)-s_j(q)=\langle q,k_i-k_j\rangle\geq\delta
\end{equation*}
for every \(j\neq i\). Equivalently,
\begin{equation*}
    s_j(q)-s_i(q)\leq-\delta,\quad j\neq i.
\end{equation*}
Define
\begin{equation*}
    R_i(q):=\sum_{j\neq i}e^{(s_j(q)-s_i(q))/\tau},\quad \alpha_{\delta,\tau}:=(N-1)e^{-\delta/\tau}.
\end{equation*}
For every \(j\neq i\),
\begin{equation*}
    e^{(s_j(q)-s_i(q))/\tau}\leq e^{-\delta/\tau}.
\end{equation*}
Summing these \(N-1\) inequalities gives
\begin{equation*}
    0\leq R_i(q)=\sum_{j\neq i}e^{(s_j(q)-s_i(q))/\tau}\leq(N-1)e^{-\delta/\tau}=\alpha_{\delta,\tau}.
\end{equation*}

Factoring \(e^{s_i(q)/\tau}\) from the softmax denominator yields
\begin{equation*}
    Z(q)=e^{s_i(q)/\tau}\left(1+\sum_{j\neq i}e^{(s_j(q)-s_i(q))/\tau}\right)=e^{s_i(q)/\tau}\bigl(1+R_i(q)\bigr).
\end{equation*}
Therefore,
\begin{equation*}
    p_i^{(\tau)}(q)=\frac{1}{1+R_i(q)}
\end{equation*}
and, for \(j\neq i\),
\begin{equation*}
    p_j^{(\tau)}(q)=\frac{e^{(s_j(q)-s_i(q))/\tau}}{1+R_i(q)}.
\end{equation*}
The total probability assigned to the nonwinning indices is consequently
\begin{equation*}
    \ell_i(q):=\sum_{j\neq i}p_j^{(\tau)}(q)=\frac{\sum_{j\neq i}e^{(s_j(q)-s_i(q))/\tau}}{1+R_i(q)}=\frac{R_i(q)}{1+R_i(q)}.
\end{equation*}
Since \(0\leq R_i(q)\leq\alpha_{\delta,\tau}\), we have
\begin{equation*}
    \frac{R_i(q)}{1+R_i(q)}\leq\frac{\alpha_{\delta,\tau}}{1+\alpha_{\delta,\tau}},
\end{equation*}
because multiplication by the positive denominators reduces this inequality to
\begin{equation*}
    R_i(q)\bigl(1+\alpha_{\delta,\tau}\bigr)\leq\alpha_{\delta,\tau}\bigl(1+R_i(q)\bigr)\iff R_i(q)\leq\alpha_{\delta,\tau}.
\end{equation*}
Thus,
\begin{equation}
\label{eq:off_mass_bound}
    \ell_i(q)=1-p_i^{(\tau)}(q)\leq\frac{\alpha_{\delta,\tau}}{1+\alpha_{\delta,\tau}}\leq\alpha_{\delta,\tau}=(N-1)e^{-\delta/\tau}.
\end{equation}

We now prove the potential-value approximation. Using the preceding factorization of \(Z(q)\),
\begin{equation*}
    P^{(\tau)}(q)=\tau\log\left(e^{s_i(q)/\tau}\bigl(1+R_i(q)\bigr)\right)=s_i(q)+\tau\log\bigl(1+R_i(q)\bigr).
\end{equation*}
Since \(q\in C_i^\delta\) and \(\delta>0\), index \(i\) is the unique maximizer, so
\begin{equation*}
    P^{(0)}(q)=\max_{1\leq j\leq N}s_j(q)=s_i(q).
\end{equation*}
It follows that
\begin{equation*}
    P^{(\tau)}(q)-P^{(0)}(q)=\tau\log\bigl(1+R_i(q)\bigr).
\end{equation*}
Because \(0\leq R_i(q)\leq\alpha_{\delta,\tau}\) and the logarithm is increasing,
\begin{equation*}
    0\leq P^{(\tau)}(q)-P^{(0)}(q)\leq\tau\log\bigl(1+\alpha_{\delta,\tau}\bigr)=\tau\log\left(1+(N-1)e^{-\delta/\tau}\right).
\end{equation*}

We next derive the gradient and Hessian. Since \(s_j(q)=q^\top k_j\),
\begin{equation*}
    \nabla s_j(q)=k_j.
\end{equation*}
Differentiating \(Z(q)\) gives
\begin{equation*}
    \nabla Z(q)=\sum_{j=1}^{N}\nabla e^{s_j(q)/\tau}=\sum_{j=1}^{N}e^{s_j(q)/\tau}\frac{k_j}{\tau}=\frac{1}{\tau}\sum_{j=1}^{N}e^{s_j(q)/\tau}k_j.
\end{equation*}
Therefore,
\begin{equation*}
    \nabla P^{(\tau)}(q)=\tau\frac{\nabla Z(q)}{Z(q)}=\frac{\sum_{j=1}^{N}e^{s_j(q)/\tau}k_j}{Z(q)}=\sum_{j=1}^{N}p_j^{(\tau)}(q)k_j.
\end{equation*}
Define
\begin{equation*}
    \mu(q):=\nabla P^{(\tau)}(q)=\sum_{j=1}^{N}p_j^{(\tau)}(q)k_j.
\end{equation*}
Since \(\sum_{j=1}^{N}p_j^{(\tau)}(q)=1\),
\begin{equation*}
    \mu(q)-k_i=\sum_{j=1}^{N}p_j^{(\tau)}(q)(k_j-k_i)=\sum_{j\neq i}p_j^{(\tau)}(q)(k_j-k_i).
\end{equation*}
By the triangle inequality and the definition of \(K_i\),
\begin{equation*}
    \|\mu(q)-k_i\|_2\leq\sum_{j\neq i}p_j^{(\tau)}(q)\|k_j-k_i\|_2\leq K_i\sum_{j\neq i}p_j^{(\tau)}(q)=K_i\ell_i(q).
\end{equation*}
Applying \eqref{eq:off_mass_bound} gives
\begin{equation*}
    \left\|\nabla P^{(\tau)}(q)-k_i\right\|_2\leq K_i(N-1)e^{-\delta/\tau}.
\end{equation*}

To calculate the Hessian, we first differentiate each softmax probability. By the quotient rule,
\begin{equation*}
    \nabla p_j^{(\tau)}(q)=\frac{\nabla e^{s_j(q)/\tau}\,Z(q)-e^{s_j(q)/\tau}\nabla Z(q)}{Z(q)^2}.
\end{equation*}
Substituting
\begin{equation*}
    \nabla e^{s_j(q)/\tau}=\frac{1}{\tau}e^{s_j(q)/\tau}k_j
\end{equation*}
and
\begin{equation*}
    \nabla Z(q)=\frac{Z(q)}{\tau}\mu(q)
\end{equation*}
yields
\begin{equation*}
    \nabla p_j^{(\tau)}(q)=\frac{e^{s_j(q)/\tau}Z(q)k_j/\tau-e^{s_j(q)/\tau}Z(q)\mu(q)/\tau}{Z(q)^2}.
\end{equation*}
Cancelling one factor of \(Z(q)\) gives
\begin{equation}
\label{eq:softmax_query_derivative}
    \nabla p_j^{(\tau)}(q)=\frac{e^{s_j(q)/\tau}}{\tau Z(q)}\bigl(k_j-\mu(q)\bigr)=\frac{p_j^{(\tau)}(q)}{\tau}\bigl(k_j-\mu(q)\bigr).
\end{equation}

Since \(\mu(q)=\sum_{j=1}^{N}p_j^{(\tau)}(q)k_j\), its Jacobian is
\begin{equation*}
    \nabla^2P^{(\tau)}(q)=\nabla\mu(q)=\sum_{j=1}^{N}k_j\bigl(\nabla p_j^{(\tau)}(q)\bigr)^\top.
\end{equation*}
Using Eq.~\eqref{eq:softmax_query_derivative},
\begin{equation*}
    \nabla^2P^{(\tau)}(q)=\frac{1}{\tau}\sum_{j=1}^{N}p_j^{(\tau)}(q)k_j\bigl(k_j-\mu(q)\bigr)^\top.
\end{equation*}
Expanding the right-hand side gives
\begin{equation*}
    \nabla^2P^{(\tau)}(q)=\frac{1}{\tau}\left(\sum_{j=1}^{N}p_j^{(\tau)}(q)k_jk_j^\top-\sum_{j=1}^{N}p_j^{(\tau)}(q)k_j\mu(q)^\top\right).
\end{equation*}
Because \(\sum_{j=1}^{N}p_j^{(\tau)}(q)k_j=\mu(q)\),
\begin{equation}
\label{eq:lse_query_hessian}
    \nabla^2P^{(\tau)}(q)=\frac{1}{\tau}\left(\sum_{j=1}^{N}p_j^{(\tau)}(q)k_jk_j^\top-\mu(q)\mu(q)^\top\right).
\end{equation}

Define
\begin{equation*}
    \xi_j:=k_j-k_i,\quad m:=\sum_{j=1}^{N}p_j^{(\tau)}(q)\xi_j=\mu(q)-k_i.
\end{equation*}
Then \(\xi_i=0\) and \(\|\xi_j\|_2\leq K_i\) for every \(j\neq i\). Moreover,
\begin{equation*}
    k_j=k_i+\xi_j,\quad \mu(q)=k_i+m.
\end{equation*}
Expanding the matrix in Eq.~\eqref{eq:lse_query_hessian} gives
\begin{equation*}
    \sum_{j=1}^{N}p_j^{(\tau)}(q)k_jk_j^\top=\sum_{j=1}^{N}p_j^{(\tau)}(q)(k_i+\xi_j)(k_i+\xi_j)^\top.
\end{equation*}
Using \(\sum_jp_j^{(\tau)}(q)=1\) and \(\sum_jp_j^{(\tau)}(q)\xi_j=m\),
\begin{equation*}
    \sum_{j=1}^{N}p_j^{(\tau)}(q)k_jk_j^\top=k_ik_i^\top+k_im^\top+mk_i^\top+\sum_{j=1}^{N}p_j^{(\tau)}(q)\xi_j\xi_j^\top.
\end{equation*}
Similarly,
\begin{equation*}
    \mu(q)\mu(q)^\top=(k_i+m)(k_i+m)^\top=k_ik_i^\top+k_im^\top+mk_i^\top+mm^\top.
\end{equation*}
Subtracting these two identities yields
\begin{equation*}
    \sum_{j=1}^{N}p_j^{(\tau)}(q)k_jk_j^\top-\mu(q)\mu(q)^\top=\sum_{j=1}^{N}p_j^{(\tau)}(q)\xi_j\xi_j^\top-mm^\top.
\end{equation*}
Lemma~\ref{lem:weighted_covariance_bound}, applied with \(p_j=p_j^{(\tau)}(q)\), \(K=K_i\), and \(\ell=\ell_i(q)\), therefore gives
\begin{equation*}
    \left\|\sum_{j=1}^{N}p_j^{(\tau)}(q)k_jk_j^\top-\mu(q)\mu(q)^\top\right\|_2\leq K_i^2\ell_i(q).
\end{equation*}
Combining this inequality with Eq.~\eqref{eq:lse_query_hessian} and \eqref{eq:off_mass_bound} gives
\begin{equation*}
    \left\|\nabla^2P^{(\tau)}(q)\right\|_2\leq\frac{K_i^2}{\tau}\ell_i(q)\leq\frac{K_i^2}{\tau}(N-1)e^{-\delta/\tau}.
\end{equation*}

We now prove the local affine approximation bounds. The set \(C_i^\delta\) is convex because it is the intersection of the closed halfspaces
\begin{equation*}
    \left\{q\in\mathbb R^{d_k}:\langle q,k_i-k_j\rangle\geq\delta\right\},\quad j\neq i.
\end{equation*}
Indeed, if \(q,q'\in C_i^\delta\) and \(t\in[0,1]\), then
\begin{equation*}
    \left\langle(1-t)q+tq',k_i-k_j\right\rangle=(1-t)\langle q,k_i-k_j\rangle+t\langle q',k_i-k_j\rangle\geq(1-t)\delta+t\delta=\delta.
\end{equation*}
Thus the line segment
\begin{equation*}
    \gamma(t):=q+t(q'-q),\quad t\in[0,1],
\end{equation*}
lies entirely in \(C_i^\delta\).

Let \(h:=q'-q\). By the fundamental theorem of calculus applied componentwise to \(\nabla P^{(\tau)}\),
\begin{equation*}
    \nabla P^{(\tau)}(q')-\nabla P^{(\tau)}(q)=\int_0^1\nabla^2P^{(\tau)}(\gamma(t))h\,dt.
\end{equation*}
Taking norms and using the Hessian bound at every \(\gamma(t)\in C_i^\delta\),
\begin{equation*}
    \left\|\nabla P^{(\tau)}(q')-\nabla P^{(\tau)}(q)\right\|_2\leq\int_0^1\left\|\nabla^2P^{(\tau)}(\gamma(t))\right\|_2\|h\|_2\,dt.
\end{equation*}
Therefore,
\begin{equation*}
    \left\|\nabla P^{(\tau)}(q')-\nabla P^{(\tau)}(q)\right\|_2\leq\int_0^1\frac{K_i^2}{\tau}(N-1)e^{-\delta/\tau}\|h\|_2\,dt=\frac{K_i^2}{\tau}(N-1)e^{-\delta/\tau}\|q'-q\|_2.
\end{equation*}

For the first-order Taylor remainder, the fundamental theorem of calculus gives
\begin{equation*}
    P^{(\tau)}(q')-P^{(\tau)}(q)=\int_0^1\left\langle\nabla P^{(\tau)}(q+th),h\right\rangle dt.
\end{equation*}
Subtracting \(\langle\nabla P^{(\tau)}(q),h\rangle\) gives
\begin{equation*}
    P^{(\tau)}(q')-P^{(\tau)}(q)-\left\langle\nabla P^{(\tau)}(q),h\right\rangle=\int_0^1\left\langle\nabla P^{(\tau)}(q+th)-\nabla P^{(\tau)}(q),h\right\rangle dt.
\end{equation*}
By the Cauchy--Schwarz inequality,
\begin{equation*}
    \left|P^{(\tau)}(q')-P^{(\tau)}(q)-\left\langle\nabla P^{(\tau)}(q),h\right\rangle\right|\leq\int_0^1\left\|\nabla P^{(\tau)}(q+th)-\nabla P^{(\tau)}(q)\right\|_2\|h\|_2\,dt.
\end{equation*}
Applying the gradient Lipschitz bound to the two points \(q\) and \(q+th\), whose distance is \(t\|h\|_2\), yields
\begin{equation*}
    \left\|\nabla P^{(\tau)}(q+th)-\nabla P^{(\tau)}(q)\right\|_2\leq\frac{K_i^2}{\tau}(N-1)e^{-\delta/\tau}t\|h\|_2.
\end{equation*}
Hence,
\begin{equation*}
    \left|P^{(\tau)}(q')-P^{(\tau)}(q)-\left\langle\nabla P^{(\tau)}(q),h\right\rangle\right|\leq\frac{K_i^2}{\tau}(N-1)e^{-\delta/\tau}\|h\|_2^2\int_0^1t\,dt.
\end{equation*}
Since \(\int_0^1t\,dt=1/2\),
\begin{equation*}
    \left|P^{(\tau)}(q')-P^{(\tau)}(q)-\left\langle\nabla P^{(\tau)}(q),q'-q\right\rangle\right|\leq\frac{K_i^2}{2\tau}(N-1)e^{-\delta/\tau}\|q'-q\|_2^2.
\end{equation*}

We next prove the attention-output approximation. Since \(\sum_{j=1}^{N}p_j^{(\tau)}(q)=1\),
\begin{equation*}
    A^{(\tau)}(q)-v_i=\sum_{j=1}^{N}p_j^{(\tau)}(q)(v_j-v_i)=\sum_{j\neq i}p_j^{(\tau)}(q)(v_j-v_i).
\end{equation*}
Using the triangle inequality and the definition of \(D_i\),
\begin{equation*}
    \left\|A^{(\tau)}(q)-v_i\right\|_2\leq\sum_{j\neq i}p_j^{(\tau)}(q)\|v_j-v_i\|_2\leq D_i\sum_{j\neq i}p_j^{(\tau)}(q)=D_i\ell_i(q).
\end{equation*}
Applying \eqref{eq:off_mass_bound} gives
\begin{equation*}
    \left\|A^{(\tau)}(q)-v_i\right\|_2\leq D_i(N-1)e^{-\delta/\tau}.
\end{equation*}

We now calculate the Jacobian of \(A^{(\tau)}\). Since
\begin{equation*}
    A^{(\tau)}(q)=\sum_{j=1}^{N}p_j^{(\tau)}(q)v_j,
\end{equation*}
differentiation gives
\begin{equation*}
    JA^{(\tau)}(q)=\sum_{j=1}^{N}v_j\bigl(\nabla p_j^{(\tau)}(q)\bigr)^\top.
\end{equation*}
Because \(\sum_{j=1}^{N}p_j^{(\tau)}(q)=1\),
\begin{equation*}
    \sum_{j=1}^{N}\nabla p_j^{(\tau)}(q)=\nabla\left(\sum_{j=1}^{N}p_j^{(\tau)}(q)\right)=\nabla1=0.
\end{equation*}
Therefore,
\begin{equation*}
    JA^{(\tau)}(q)=\sum_{j=1}^{N}(v_j-v_i)\bigl(\nabla p_j^{(\tau)}(q)\bigr)^\top.
\end{equation*}
Substituting \eqref{eq:softmax_query_derivative} gives
\begin{equation*}
    JA^{(\tau)}(q)=\frac{1}{\tau}\sum_{j=1}^{N}p_j^{(\tau)}(q)(v_j-v_i)\bigl(k_j-\mu(q)\bigr)^\top.
\end{equation*}
The term with \(j=i\) is zero because \(v_i-v_i=0\), so
\begin{equation*}
    JA^{(\tau)}(q)=\frac{1}{\tau}\sum_{j\neq i}p_j^{(\tau)}(q)(v_j-v_i)\bigl(k_j-\mu(q)\bigr)^\top.
\end{equation*}

For every \(j\neq i\),
\begin{equation*}
    k_j-\mu(q)=(k_j-k_i)-(\mu(q)-k_i).
\end{equation*}
Hence,
\begin{equation*}
    \|k_j-\mu(q)\|_2\leq\|k_j-k_i\|_2+\|\mu(q)-k_i\|_2\leq K_i+K_i\ell_i(q)=K_i\bigl(1+\ell_i(q)\bigr).
\end{equation*}
Since \(0\leq\ell_i(q)\leq1\),
\begin{equation*}
    \|k_j-\mu(q)\|_2\leq2K_i.
\end{equation*}

For vectors \(a\) and \(b\), the spectral norm of the rank-one matrix \(ab^\top\) is
\begin{equation*}
    \|ab^\top\|_2=\sup_{\|x\|_2=1}\|ab^\top x\|_2=\|a\|_2\sup_{\|x\|_2=1}|b^\top x|=\|a\|_2\|b\|_2.
\end{equation*}
Therefore,
\begin{equation*}
    \left\|JA^{(\tau)}(q)\right\|_2\leq\frac{1}{\tau}\sum_{j\neq i}p_j^{(\tau)}(q)\|v_j-v_i\|_2\|k_j-\mu(q)\|_2.
\end{equation*}
Using \(\|v_j-v_i\|_2\leq D_i\) and \(\|k_j-\mu(q)\|_2\leq K_i(1+\ell_i(q))\),
\begin{equation*}
    \left\|JA^{(\tau)}(q)\right\|_2\leq\frac{D_iK_i}{\tau}\bigl(1+\ell_i(q)\bigr)\sum_{j\neq i}p_j^{(\tau)}(q)=\frac{D_iK_i}{\tau}\ell_i(q)\bigl(1+\ell_i(q)\bigr).
\end{equation*}
Since \(0\leq\ell_i(q)\leq1\),
\begin{equation*}
    \ell_i(q)\bigl(1+\ell_i(q)\bigr)\leq2\ell_i(q).
\end{equation*}
Combining this inequality with Eq.~\eqref{eq:off_mass_bound} gives
\begin{equation*}
    \left\|JA^{(\tau)}(q)\right\|_2\leq\frac{2D_iK_i}{\tau}\ell_i(q)\leq\frac{2D_iK_i}{\tau}(N-1)e^{-\delta/\tau}.
\end{equation*}

Finally, let \(q,q'\in C_i^\delta\), set \(h:=q'-q\), and define \(\gamma(t):=q+th\). Since \(C_i^\delta\) is convex, \(\gamma(t)\in C_i^\delta\) for every \(t\in[0,1]\). By the fundamental theorem of calculus applied componentwise to \(A^{(\tau)}\),
\begin{equation*}
    A^{(\tau)}(q')-A^{(\tau)}(q)=\int_0^1JA^{(\tau)}(\gamma(t))h\,dt.
\end{equation*}
Taking norms gives
\begin{equation*}
    \left\|A^{(\tau)}(q')-A^{(\tau)}(q)\right\|_2\leq\int_0^1\left\|JA^{(\tau)}(\gamma(t))\right\|_2\|h\|_2\,dt.
\end{equation*}
Applying the Jacobian bound at every point \(\gamma(t)\in C_i^\delta\) yields
\begin{equation*}
    \left\|A^{(\tau)}(q')-A^{(\tau)}(q)\right\|_2\leq\int_0^1\frac{2D_iK_i}{\tau}(N-1)e^{-\delta/\tau}\|h\|_2\,dt.
\end{equation*}
Evaluating the integral gives
\begin{equation*}
    \left\|A^{(\tau)}(q')-A^{(\tau)}(q)\right\|_2\leq\frac{2D_iK_i}{\tau}(N-1)e^{-\delta/\tau}\|q'-q\|_2.
\end{equation*}
This proves all assertions.
\end{proof}
}

\end{document}